\documentclass[11pt]{article}

\usepackage[margin=1in]{geometry}
\usepackage{natbib}

\usepackage{booktabs}
\usepackage{multirow}
\usepackage{amsmath}
\usepackage{amssymb}
\usepackage{amsthm}
\usepackage{graphicx}
\usepackage{algorithm}
\usepackage{algorithmic}
\usepackage{xspace}
\usepackage{subcaption}
\usepackage{enumitem}
\usepackage{tikz}
\usetikzlibrary{shapes.geometric}
\usepackage{hyperref}
\usepackage{xcolor}
\usepackage{url}

\hypersetup{
  pdftitle={GEODE: Angle-Adaptive OOD Detection with Universal Scorer Compatibility},
  pdfauthor={Bruno Abrahao},
  pdfsubject={Out-of-distribution detection, neural collapse, training-based methods},
  pdfkeywords={out-of-distribution detection, neural collapse, outlier exposure, OOD, machine learning},
  colorlinks=true,
  linkcolor=blue,
  citecolor=blue,
  urlcolor=blue
}

\newtheorem{theorem}{Theorem}

\newtheorem{proposition}[theorem]{Proposition}

\newcommand{\method}{GEODE\xspace}
\newcommand{\loss}{\mathcal{L}_{\mathrm{GEO}}}
\newcommand{\cosim}{\mathrm{cos}}
\newcommand{\norm}[1]{\lVert #1 \rVert}
\newcommand{\reals}{\mathbb{R}}
\newcommand{\expect}{\mathbb{E}}

\DeclareMathOperator{\nul}{\mathcal{N}}
\DeclareMathOperator{\col}{col}

\newenvironment{ack}{\section*{Acknowledgements}}{}

\title{GEODE: Angle-Adaptive OOD Detection with Universal Scorer Compatibility}

\author{
  Bruno Abrahao\\
  NYU Shanghai \\
  Leonard N.\ Stern School of Business, NYU \\
  \texttt{abrahao@nyu.edu}
}

\date{}

\begin{document}
\maketitle

\begin{abstract}
Outlier Exposure (OE) is among the strongest training-based OOD detectors on standard benchmarks but exhibits scorer-dependent tradeoffs (e.g., strong on MSP, weak on KNN) and requires curated auxiliary data. We show why OE works: its features sit at the same geometric locus as real near-OOD data, with the boundary-adjacent quartile driving nearly all of OE's gain. \textbf{OE is boundary calibration, not OOD coverage.} \method (GEOmetry-preserving DEtection) replicates this calibration synthetically through an angle-adaptive norm loss in which targets scale per-sample with cosine similarity to the nearest class mean, preserving feature geometry where boundary structure matters. Four theorems grounded in neural collapse justify the design. \method works across all seven standard scorers on CIFAR-10 (near-OOD AUROC $89.0$--$92.3$, far-OOD reaching $93.05$; no catastrophic failure on any scorer). Since the OOD regime is unknown at deployment, this is the test that matters. \method outperforms vanilla CE at matched epoch counts. Combined with OE, \method reaches $95.0$ MSP / $94.8$ KNN on CIFAR-10 and beats OE on every scorer on CIFAR-100. The gains hold on WRN-28-10 ($+4.5$ Energy, 3 seeds). Unlike methods that push OOD into the classifier null space (e.g., PFS, $14.38$ KNN AUROC, worse than random), \method's adaptive target preserves the geometry that distance-based scorers depend on.
\end{abstract}

\section{Introduction}
\label{sec:intro}

Out-of-distribution (OOD) detection flags inputs a deployed classifier was not trained on~\citep{hendrycks2017baseline, zhang2024openoodv15}. Post-hoc scoring functions, from logit-based methods like MSP~\citep{hendrycks2017baseline} and Energy~\citep{liu2020energy} to feature-space methods like KNN~\citep{sun2022knnood}, can be applied to any trained model, and training-based methods reshape features or logits to improve detection~\citep{hendrycks2019oe, wei2022logitnorm, du2022vos, ming2023cider, tao2023npos}. The two advances are in tension: training methods that boost logit-based scorers routinely degrade feature-space scorers, and vice versa. No existing training method achieves strong detection across all standard scorers \emph{and} across both near- and far-OOD regimes simultaneously. At deployment, where the OOD type is unknown a priori, this leaves practitioners exposed to scorer-and-regime mismatch.

\begin{figure*}[t]
\centering
\includegraphics[width=\textwidth]{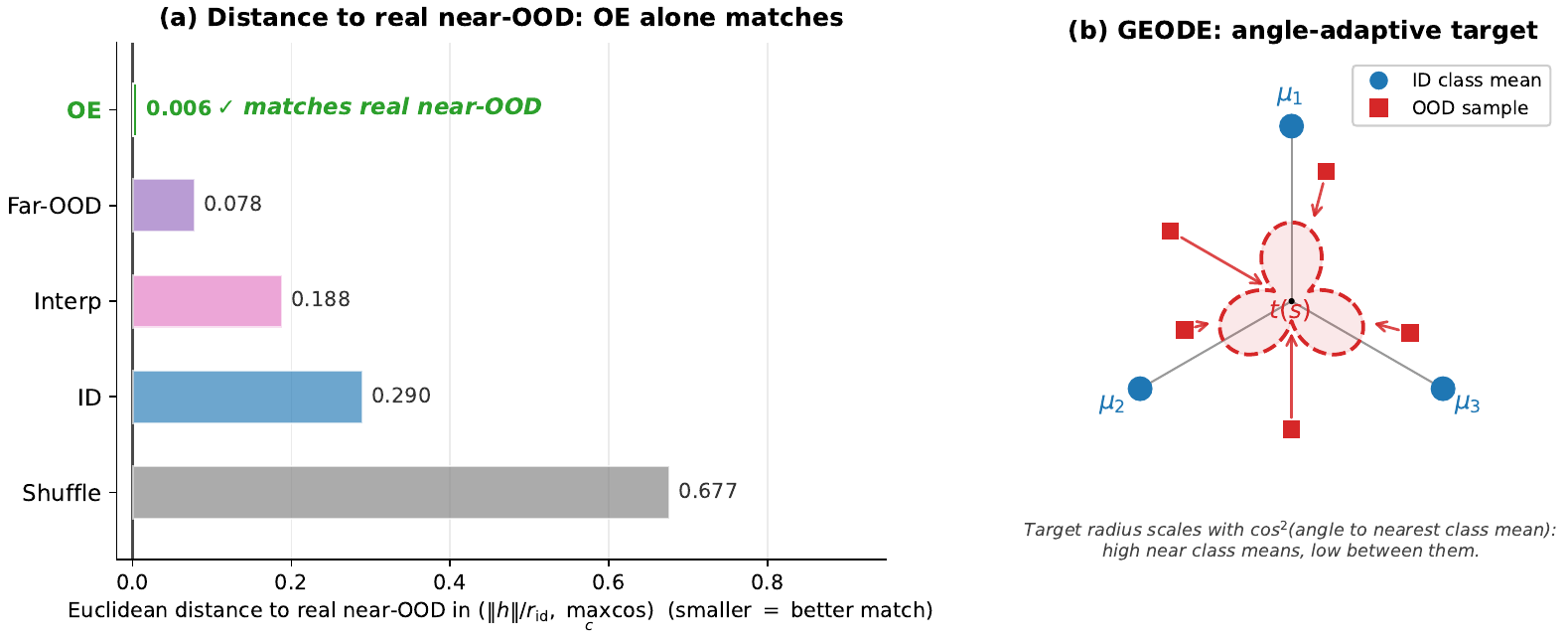}
\caption{%
\textbf{(a)}~Distance of each candidate source from real near-OOD geometry (ResNet-18, CIFAR-10), measured as Euclidean distance in the $(\norm{h}/r_{\mathrm{id}},\,\max_c\cosim)$ plane. \textbf{OE alone matches} (distance $0.006$; the next closest source is $13\times$ further). Including the projection ratio and perpendicular norm columns of Table~\ref{tab:oe_geometry} preserves the same ranking under any reasonable normalization. Candidate raw synthesis sources (Interp from prior work; Shuffle introduced here) miss the calibration zone in their unadjusted form; \method's adaptive target steers Shuffle features to the boundary calibration locus directly.
\textbf{(b)}~\method's angle-adaptive target locus traces a multi-petal pattern (dashed red): at each angle, the target radius $t(s) = s^2\alpha\,r_{\mathrm{id}}$ scales with $s = \max_c \cosim(h, \mu_c)$, the cosine similarity to the nearest class mean. Synthetic OOD samples (red squares) receive per-sample targets along this locus (red arrows): OOD features near a class mean are pulled to a relatively large radius (preserving local feature geometry that KNN-style scorers rely on), while OOD features far from any class are pulled toward the origin (sharpening logit-based separation). The per-sample, angle-coupled target is the central mechanism of \method.%
}
\label{fig:overview}
\end{figure*}

Outlier Exposure~\citep[OE;][]{hendrycks2019oe} remains the dominant training-based approach: exposing the model to real auxiliary OOD data outperforms methods that rely on synthetic outliers~\citep{zhang2024openoodv15}. But OE requires a curated auxiliary dataset, and its effectiveness depends on how well that dataset matches the true test-time OOD distribution, a requirement that is both labor-intensive and fragile. Synthetic alternatives (VOS, NPOS) improve specific logit-based scorers but fall short of OE's overall performance. PFS~\citep{wu2025pfs} pushes OOD features into the null space of the classifier, achieving strong logit scores ($93.75$ MSP) but catastrophically destroying KNN ($14.38$ AUROC, worse than random). We use PFS throughout as the named instance of Theorem~\ref{thm:nullspace_knn}, which characterizes this null-space failure mode. \emph{What geometric property of auxiliary data makes OE effective, and can we replicate it without real OOD data?}

In this paper, we answer this question and propose \method (GEOmetry-preserving DEtection). We begin with a geometric analysis of OE in the backbone's penultimate layer: OE features are geometrically indistinguishable from real near-OOD data, matching in norm ratio ($0.833$ vs $0.836$), maximum cosine similarity ($0.714$ vs $0.719$), and subspace projection ($0.952$ vs $0.953$). The boundary-adjacent quartile of OE drives almost all the gain ($94.45$ AUROC from less than half the examples). \textbf{OE is a boundary calibration mechanism}, not an OOD coverage mechanism. The geometric locus that makes OE effective is already present in any trained backbone's feature space. \method generates synthetic training signal at that locus. An angle-adaptive norm loss sets OOD norm targets proportional to cosine similarity to the nearest class mean to preserve boundary geometry; a contrastive term pushes ID norms up to widen the margin from both sides (Figure~\ref{fig:overview}). The result is a norm gap that supports \emph{every} standard scorer at competitive AUROC ($89.0$--$92.3$ across all seven on CIFAR-10 near-OOD) without catastrophic failure on any; the small KNN cost ($1.6$ points vs Vanilla) is a deliberate trade for gains on logit-side scorers (Section~\ref{sec:tradeoff_analysis}). Because \method operates in norm space rather than logit space, it is compatible with the natural convergence of neural collapse~\citep{papyan2020neural}: it outperforms vanilla CE at matched epoch counts, while LogitNorm, which modifies the logit geometry, interferes with NC convergence and degrades with longer training.

\paragraph{Contributions.}
\begin{enumerate}
    \item \textbf{Framing.} We present the first geometric characterization of why OE works (boundary calibration, not OOD coverage), diagnose PFS's catastrophic KNN failure ($14.38$ AUROC) via null-space analysis, and identify a distributional shortcut inflating OE's reported CIFAR-100 gains.
    \item \textbf{Method.} This analysis motivates \method, an angle-adaptive norm loss with a contrastive ID margin. Four theoretical results grounded in neural collapse justify the design: three establish that adaptive targeting preserves feature geometry (Theorems~\ref{thm:permutation_mean}--\ref{thm:metric}), and a fourth proves that the competing null-space approach destroys KNN (Theorem~\ref{thm:nullspace_knn}).
    \item \textbf{Empirical.} On CIFAR-10 (3 seeds), \method achieves $92.30 \pm 0.34$ near-OOD AUROC (EBO) and detects across all seven standard scorers (range $89.0$--$92.3$, no catastrophic failure on any), without auxiliary data. Combined with OE, \method reaches $95.0$ MSP / $94.8$ KNN on CIFAR-10 and beats OE on all three scorers (MSP, Energy, KNN) on CIFAR-100. \method outperforms PFS on every scorer across 3 datasets and vanilla CE at every matched epoch count; LogitNorm degrades with longer training, evidence of NC-compatible design.
\end{enumerate}

\section{Preliminaries}
\label{sec:prelim}

\paragraph{Problem setup.}
A neural network maps input $\mathbf{x} \in \reals^d$ to features $\mathbf{h} = g_\theta(\mathbf{x}) \in \reals^p$, followed by a linear classifier $W \in \reals^{C \times p}$ producing logits $\mathbf{z} = W\mathbf{h}$. Training minimizes cross-entropy on the ID dataset.

\paragraph{OOD scoring functions.}
At test time, a scoring function $S(\mathbf{x})$ separates ID from OOD inputs. Standard scorers fall into two geometric families, which we instantiate with three widely used representatives:
\begin{itemize}[nosep]
    \item \textbf{MSP}~\citep{hendrycks2017baseline}: $S_{\mathrm{MSP}}(\mathbf{x}) = \max_c \mathrm{softmax}(z_c)$. Operates on logits, like ODIN~\citep{liang2018enhancing}, ReAct~\citep{sun2021react}, ASH~\citep{djurisic2023ash}, and Scale~\citep{xu2024scale}.
    \item \textbf{Energy}~\citep{liu2020energy}: $S_{\mathrm{Energy}}(\mathbf{x}) = \log \sum_{c} \exp(z_c)$. Theoretically grounded as a free energy.
    \item \textbf{KNN}~\citep{sun2022knnood}: $S_{\mathrm{KNN}}(\mathbf{x}) = -\norm{\mathbf{h} - \mathbf{h}^{(k)}}$. Operates in feature space, like Mahalanobis~\citep{lee2018mahalanobis}.
\end{itemize}
Logit-based scorers depend on magnitude and direction of logits; feature-space scorers depend on local geometric structure. This distinction is central to the tradeoff we address.

\paragraph{Neural collapse and the ETF.}
In the terminal phase of training, \citet{papyan2020neural} observed that class means converge to a simplex equiangular tight frame (ETF): the $C$ class mean directions $\{\hat{\boldsymbol{\mu}}_c\}_{c=1}^C$ satisfy
\begin{equation}
    \hat{\boldsymbol{\mu}}_i^\top \hat{\boldsymbol{\mu}}_j = \frac{C\delta_{ij} - 1}{C - 1},
    \label{eq:etf}
\end{equation}
where $\hat{\boldsymbol{\mu}}_c = (\boldsymbol{\mu}_c - \boldsymbol{\mu}_G)/\norm{\boldsymbol{\mu}_c - \boldsymbol{\mu}_G}$ are the centered and normalized class means and $\boldsymbol{\mu}_G = \frac{1}{C}\sum_c \boldsymbol{\mu}_c$ is the global mean. Under neural collapse, classifier weight rows $\mathbf{w}_c$ align with class mean directions, so logits $z_c = \mathbf{w}_c^\top \mathbf{h}$ depend only on the projection of $\mathbf{h}$ onto the $C$-dimensional classifier subspace $\mathrm{col}(W^\top)$. Features in the $(p - C)$-dimensional null space of $W$ are invisible to any logit-derived scorer, yet carry the norm signal that feature-space scorers rely on.

\paragraph{Score-family tradeoff.}
We formalize the score-family tradeoff in Appendix~\ref{app:tradeoff_proposition} (Proposition~\ref{prop:tradeoff}): Sep-type scores (normalized-logit scorers) are provably near-chance on near-OOD under NC, while Radius scores fail on far-OOD by null-space concentration~\citep{wang2025impossibility}. No single standard scorer can simultaneously detect both regimes---the tradeoff we observe empirically and address with \method. Rather than designing a new composite score, \method imposes an angle-adaptive norm target during training, producing a universal norm gap at the decision boundary that benefits \emph{every} standard scorer. We synthesize OOD features by independently permuting each dimension across the batch (details in Appendix~\ref{app:permutation}).

\section{Method}
\label{sec:method}

\subsection{OE Geometric Analysis}
\label{sec:oe_analysis}

The geometric locus that makes OE effective is already present in any trained backbone's feature space. The question is whether we can identify it precisely enough to generate synthetic training signal there. We begin by characterizing the geometric properties of OE features in the backbone's penultimate layer. We extract 512-d features (ResNet-18) for five data sources through a vanilla CIFAR-10 backbone: ID training data, OE auxiliary data (TIN-597), real near-OOD test data, our feature shuffling, and cross-class interpolation~\citep{zhang2018mixup, verma2019manifoldmixup, zhang2023mixoe, wang2025bootood} as a comparator. For each feature $\mathbf{h}$, we compute: centered norm $\norm{\mathbf{h} - \boldsymbol{\mu}_G}$, maximum cosine similarity to any class mean $s = \max_c \cosim(\mathbf{h}, \boldsymbol{\mu}_c)$, and projection onto the ID subspace.

\paragraph{OE $\approx$ near-OOD geometrically.}
Table~\ref{tab:oe_geometry} reveals that OE features are geometrically indistinguishable from real near-OOD: centered norm ($4.27$ vs $4.29$), max cosine ($0.714$ vs $0.719$), and projection ratio ($0.952$ vs $0.953$) all agree to within $0.5\%$.

\begin{table}[t]
\centering
\caption{Geometric characterization of feature sources in vanilla ResNet-18 feature space (CIFAR-10). OE is geometrically indistinguishable from real near-OOD. \method's feature shuffling produces features in a distinct geometric region; cross-class interpolation is included for comparison. Although shuffle starts far from OE, the adaptive target (Section~\ref{sec:adaptive}) steers each sample to a per-sample point on the OE-equivalent calibration locus, decoupling the synthesis source from the training-time geometry.}
\label{tab:oe_geometry}
\small
\begin{tabular}{@{}l ccccc@{}}
\toprule
Source & Norm$/r_{\mathrm{ref}}$ & Max cos & Proj.\ ratio & Perp norm \\
\midrule
ID train       & 1.000 & 0.958 & 0.981 & 0.96 \\
\textbf{OE (TIN-597)} & \textbf{0.833} & \textbf{0.714} & \textbf{0.952} & \textbf{1.27} \\
\textbf{Near-OOD}     & \textbf{0.836} & \textbf{0.719} & \textbf{0.953} & \textbf{1.24} \\
Far-OOD        & 0.761 & 0.699 & 0.949 & 1.37 \\
Shuffle (ours) & 1.100 & 0.096 & 0.171 & 5.55 \\
Interp.        & 0.859 & 0.906 & 0.980 & 0.84 \\
\bottomrule
\end{tabular}
\end{table}

\paragraph{Which OE examples are most effective?}
We partition OE into cosine-similarity bins and train OE separately with each subset (bin counts are unequal because cosine values are not uniformly distributed across OE). Table~\ref{tab:oe_quartile} shows that the \emph{ID-like} examples (cosine $> 0.85$) are the most effective, achieving $94.45$ near-OOD AUROC with less than half the data, while \emph{moderate} examples (cosine $0.5$--$0.7$) barely improve over vanilla ($88.82$ vs $88.03$). \textbf{OE is a boundary calibration mechanism}: the examples closest to the ID decision boundary provide the strongest training signal.

\begin{table}[t]
\centering
\caption{OE cosine-bin ablation on CIFAR-10. Bin sizes are unequal because OE cosine values are not uniformly distributed. Training OE with only the ID-like bin (cos $> 0.85$, $13{,}035$ samples) nearly matches full OE. Far-from-ID examples barely help. The effect cannot fully isolate cosine regime from sample size; we report it as qualitative evidence for the boundary-calibration hypothesis.}
\label{tab:oe_quartile}
\small
\begin{tabular}{@{}l cccccc@{}}
\toprule
OE subset & Cos range & $N$ & AUROC$\uparrow$ & FPR$\downarrow$ & ID-ACC \\
\midrule
No OE (vanilla) & --- & 0 & 88.03 & 48.18 & 95.22 \\
Q2 (moderate) & $[0.5, 0.7)$ & 1,289 & 88.82 & 38.83 & 87.93 \\
Q3 (near-OOD) & $[0.7, 0.85)$ & 15,526 & 93.99 & 23.96 & 93.68 \\
\textbf{Q4 (ID-like)} & $\mathbf{[0.85, 1.0)}$ & \textbf{13,035} & \textbf{94.45} & \textbf{20.36} & \textbf{94.41} \\
Full OE & $[0, 1.0)$ & 29,850 & 94.14 & 19.65 & 94.24 \\
\bottomrule
\end{tabular}
\end{table}

\subsection{Angle-Adaptive Norm Loss}
\label{sec:adaptive}

If effective OOD supervision creates a norm gap at the ID boundary, we can design a training loss that achieves this synthetically. Following \citet{wang2025bootood}, we maintain exponential moving averages (EMA) of the class means $\{\boldsymbol{\mu}_c\}_{c=1}^{C}$ and the average ID feature norm $r_{\mathrm{id}}$ throughout training:
\begin{align}
    \boldsymbol{\mu}_c &\leftarrow \beta \boldsymbol{\mu}_c + (1 - \beta) \bar{\mathbf{h}}_c, \label{eq:ema_mean} \\
    r_{\mathrm{id}} &\leftarrow \beta \, r_{\mathrm{id}} + (1 - \beta) \overline{\norm{\mathbf{h}}}, \label{eq:ema_norm}
\end{align}
where $\beta = 0.99$ is the EMA momentum.

\paragraph{Adaptive target norm.}
For each synthetic OOD feature $\mathbf{h}^{\mathrm{ood}}$, we compute its maximum cosine similarity to any class mean:
\begin{equation}
    s = \max_{c \in \{1, \ldots, C\}} \cosim\!\left(\mathbf{h}^{\mathrm{ood}}, \boldsymbol{\mu}_c\right).
    \label{eq:max_cos}
\end{equation}
The target norm is then:
\begin{equation}
    t = s^2 \cdot \alpha \cdot r_{\mathrm{id}},
    \label{eq:target}
\end{equation}
where $\alpha \in (0, 1]$ controls the maximum relaxation (default $\alpha = 0.2$). When $s \approx 1$ (near a class mean), $t \approx \alpha \cdot r_{\mathrm{id}}$: the feature retains its position near the class boundary. When $s \approx 0$ (far from all classes), $t \approx 0$: the feature is pushed to the origin for maximum logit-based discrimination.

\paragraph{Contrastive ID norm encouragement.}
To widen the ID--OOD margin from both sides, we add a contrastive term that pushes ID norms \emph{up}:
\begin{equation}
    \mathcal{L}_{\mathrm{id}} = \frac{1}{B} \sum_{i=1}^{B} \max\!\left(0,\; r_{\mathrm{id}} - \norm{\mathbf{h}_i}\right)^2,
    \label{eq:id_norm}
\end{equation}
which penalizes ID features only when their norm falls below $r_{\mathrm{id}}$ (hinge loss). The OE geometric analysis above motivates this two-sided margin: the most effective OOD training signal comes from features near the ID boundary, and a wider norm gap at that boundary improves detection across all scorer families.

\paragraph{Full training objective.}
The \method loss combines cross-entropy, the adaptive OOD norm penalty, and the contrastive ID margin:
\begin{equation}
    \loss = \mathcal{L}_{\mathrm{CE}}(\mathbf{z}, y) + \mathcal{L}_{\mathrm{ood}} + \lambda_{\mathrm{id}} \cdot \mathcal{L}_{\mathrm{id}},
    \label{eq:full_loss}
\end{equation}
where $\mathcal{L}_{\mathrm{ood}} = \frac{1}{B'} \sum_{j=1}^{B'} (\norm{\mathbf{h}_j^{\mathrm{ood}}} - t_j)^2$ is the adaptive norm penalty over $B'$ synthetic OOD features (target $t_j$ from Eq.~\ref{eq:target}) and $\mathcal{L}_{\mathrm{id}}$ is the contrastive hinge of Eq.~\ref{eq:id_norm}. We use $\lambda_{\mathrm{id}} = 0.5$ on CIFAR-10 and $\lambda_{\mathrm{id}} = 0$ on CIFAR-100 and ImageNet-200, where the weaker neural collapse geometry cannot support the contrastive margin (Appendix~\ref{app:contrastive_c100}). The full training procedure (Algorithm~\ref{alg:geode}), implementation details, and computational cost analysis are in Appendix~\ref{app:training}.

\subsection{Theoretical Justification}
\label{sec:theory}

Four results grounded in neural collapse justify the \method design. All proofs are in Appendix~\ref{app:proofs}.

\begin{theorem}[Shuffled Feature Mean under Neural Collapse]
\label{thm:permutation_mean}
Let $\{\boldsymbol{\mu}_c\}_{c=1}^C$ be class means forming a simplex ETF centered at global mean $\boldsymbol{\mu}_G = \frac{1}{C}\sum_{c=1}^C \boldsymbol{\mu}_c$. Let $\mathbf{h}^{\mathrm{ood}}$ be obtained by independent dimension-wise permutation of a balanced batch of ID features drawn from all $C$ classes. Then:
\begin{enumerate}[nosep]
    \item $\expect[\mathbf{h}^{\mathrm{ood}}] = \boldsymbol{\mu}_G$, and
    \item $\boldsymbol{\mu}_G \perp (\boldsymbol{\mu}_c - \boldsymbol{\mu}_G)$ for all $c \in \{1,\ldots,C\}$, i.e.\ $\cosim(\boldsymbol{\mu}_G,\, \boldsymbol{\mu}_c - \boldsymbol{\mu}_G) = 0$.
\end{enumerate}
\end{theorem}
\noindent\textit{Intuition:} Shuffled outliers concentrate at the global mean. Because $\boldsymbol{\mu}_G$ is orthogonal to every centered class direction, its cosine similarity to the nearest class mean is small (vanishing as $C\to\infty$), justifying near-zero norm targeting on average; individual shuffled samples nevertheless vary widely in cosine similarity to the nearest class mean, motivating the per-sample adaptive target.

\begin{theorem}[Geometry Preservation]
\label{thm:geometry}
Let $\mathbf{h}^{\mathrm{ood}} \in \reals^p$ be a synthetic OOD feature with cosine similarity $s = \max_c \cosim(\mathbf{h}^{\mathrm{ood}}, \boldsymbol{\mu}_c)$ to the nearest class mean, and let $r_{\mathrm{id}} > 0$ be the average ID feature norm. For any $\alpha \in (0,1)$ and any feature with $s > 0$ and $\norm{\mathbf{h}^{\mathrm{ood}}} \geq s^2 \cdot \alpha \cdot r_{\mathrm{id}}$:
\begin{equation}
    \Delta_{\mathrm{adap}} = \left(\norm{\mathbf{h}^{\mathrm{ood}}} - s^2 \alpha r_{\mathrm{id}}\right)^2 < \norm{\mathbf{h}^{\mathrm{ood}}}^2 = \Delta_{\mathrm{unif}}.
\end{equation}
The adaptive target preserves $O(s^2 \cdot \alpha \cdot r_{\mathrm{id}})$ of the original norm for near-class features ($s \to 1$), while recovering the uniform penalty for far-from-class features ($s \to 0$).
\end{theorem}
\noindent\textit{Intuition:} The adaptive penalty produces strictly less geometric distortion than the uniform alternative, preserving the boundary region that KNN relies on.

\begin{theorem}[Metric Preservation under Adaptive Targeting]
\label{thm:metric}
Under the adaptive target, the pairwise distance distortion between an ID feature and a synthetic OOD feature is bounded by $r_0 - s^2 \alpha\, r_{\mathrm{id}}$, compared to distortion $r_0$ under the uniform target. For near-boundary features ($s \to 1$, $r_0 \approx r_{\mathrm{id}}$), this reduces distortion by a factor of $(1 - \alpha)$.
\end{theorem}
\noindent\textit{Intuition:} Adaptive targeting preserves local pairwise distances near class boundaries, which is why KNN detection performance is maintained.

\begin{theorem}[Null-Space Push Eliminates Class-Conditional Distance Variation]
\label{thm:nullspace_knn}
Let $\mathbf{h}_{\mathrm{ID}} \in \col(W^\top)$ and $\mathbf{h}^{\mathrm{ood}} \in \nul(W)$ with norms $r_{\mathrm{id}}$ and $r_{\mathrm{ood}}$. Then $d(\mathbf{h}_{\mathrm{ID}},\, \mathbf{h}^{\mathrm{ood}}) = \sqrt{r_{\mathrm{id}}^2 + r_{\mathrm{ood}}^2}$, independent of class identity and angular position. If $r_{\mathrm{ood}} \approx r_{\mathrm{id}}$, all ID--OOD distances concentrate at $\sqrt{2}\, r_{\mathrm{id}}$, removing the class-conditional distance variation that KNN exploits to discriminate OOD from ID.
\end{theorem}
\noindent\textit{Intuition:} Methods like PFS that push OOD into the null space of $W$ collapse the class-conditional structure of ID--OOD distances. Whether KNN actually fails empirically also depends on the within-class tightness of the ID training-set neighborhoods; Table~\ref{tab:geometry_decomp} confirms that the predicted condition holds and that KNN does collapse in practice (PFS retains only $7.3\%$ of OOD energy in $\mathrm{col}(W^\top)$, yielding KNN AUROC of $14.38$).

\section{Experiments}
\label{sec:experiments}

\subsection{Setup}
\label{sec:setup}

We evaluate on the OpenOOD v1.5 benchmark \citep{zhang2024openoodv15} with CIFAR-10, CIFAR-100 \citep{krizhevsky2009cifar}, and ImageNet-200 \citep{zhang2024openoodv15} as ID datasets. We use ResNet-18 \citep{he2016resnet} as the primary backbone (WRN-28-10 and ViT-B/16 in Appendix~\ref{app:backbones}). We report AUROC ($\uparrow$) and FPR@95 ($\downarrow$) with MSP, Energy, ODIN, and KNN scorers. Results are averaged over 3 seeds where indicated ($\pm$ std); remaining entries are single-seed. \method uses default hyperparameters: $\rho = 0.05$, $\alpha = 0.2$, and $\lambda_{\mathrm{id}} = 0.5$ on CIFAR-10 ($\lambda_{\mathrm{id}} = 0$ on CIFAR-100 and ImageNet-200; see Appendix~\ref{app:contrastive_c100}). Baselines cover the main training-based design paradigms in the OpenOOD v1.5 benchmark: logit normalization (LogitNorm), synthetic outlier generation (VOS, NPOS), contrastive feature learning (CIDER), auxiliary self-supervision (ConfBranch, RotPred), real auxiliary data (OE, MixOE), and null-space training (PFS~\citep{wu2025pfs}, reimplemented in our framework with TIN as auxiliary data). Appendix~\ref{app:coverage} maps every OpenOOD v1.5 training method to a paradigm and gives the inclusion or exclusion rationale. All baselines in Tables~\ref{tab:cifar10_nearood}--\ref{tab:cifar100_nearood} are reproduced by us using the OpenOOD codebase (3 seeds each).

\subsection{Main Results}
\label{sec:main_results}

\paragraph{CIFAR-10 near-OOD.}
Table~\ref{tab:cifar10_nearood} presents near-OOD detection results. \method achieves $92.30 \pm 0.34$ AUROC (EBO) and works with all seven standard scorers (range $89.0$--$92.3$, Table~\ref{tab:scorer_robustness}). KNN ($89.04$) is slightly below vanilla ($90.64$): the adaptive norm loss reduces OOD norms in the boundary region, which benefits logit-based scorers at a small cost to local feature-space distances. The gap is $1.6$ points, a deliberate tradeoff versus PFS's $76$-point KNN collapse. With OE, the tradeoff disappears: OE+\method achieves $95.02$ MSP \emph{and} $94.79$ KNN ($+5.86$ KNN over OE alone). Far-OOD results follow the same pattern (Appendix Table~\ref{tab:farood}): \method improves over matched-epoch vanilla on CIFAR-10 ($90.95$ vs $90.57$ MSP) and CIFAR-100 ($79.70$ vs $76.49$ MSP at 200 epochs), and OE+\method reaches $96.81$ far-OOD MSP.

\begin{table}[t]
\centering
\caption{Near-OOD detection on CIFAR-10 (ResNet-18, 200 epochs, 3 seeds). AUROC ($\uparrow$) / FPR@95 ($\downarrow$). Best without auxiliary data in \textbf{bold}; best overall \underline{underlined}.}
\label{tab:cifar10_nearood}
\small
\begin{tabular}{@{}l cc cc cc cc@{}}
\toprule
& \multicolumn{2}{c}{MSP} & \multicolumn{2}{c}{EBO} & \multicolumn{2}{c}{Scale} & \multicolumn{2}{c}{KNN} \\
\cmidrule(lr){2-3} \cmidrule(lr){4-5} \cmidrule(lr){6-7} \cmidrule(lr){8-9}
Method & AUC & FPR & AUC & FPR & AUC & FPR & AUC & FPR \\
\midrule
\multicolumn{9}{l}{\textit{Post-hoc (applied to vanilla model)}} \\
Vanilla & 88.03 & 48.18 & 87.58 & 61.32 & 82.55 & 80.45 & 90.64 & 34.00 \\
\midrule
\multicolumn{9}{l}{\textit{Training methods (no auxiliary OOD data)}} \\
ConfBranch    & 87.56 & 49.11 & 87.60 & 50.61 & --- & --- & 89.90 & 37.11 \\
RotPred       & 92.54 & 27.70 & 91.67 & 34.38 & --- & --- & 89.67 & 39.53 \\
LogitNorm     & 89.74 & 41.01 & 91.04 & 32.97 & 90.50 & 35.03 & 89.68 & 35.17 \\
VOS           & 87.62 & 49.66 & 87.67 & 49.67 & 82.92 & 77.98 & 90.24 & 34.88 \\
CIDER         & 78.35 & 77.73 & 89.83 & 42.38 & --- & --- & 91.68 & 27.24 \\
NPOS          & 87.62 & 50.32 & 88.36 & 46.49 & --- & --- & 91.38 & 28.63 \\
\textbf{\method (ours)} & \textbf{90.30} & \textbf{30.77} & \textbf{92.30} & \textbf{28.67} & \textbf{92.19} & \textbf{29.04} & 89.04 & 43.82 \\
\midrule
\multicolumn{9}{l}{\textit{Training methods (with auxiliary OOD data)}} \\
OE            & 94.91 & 19.71 & 93.44 & 24.90 & 78.49 & 80.33 & 88.93 & 42.16 \\
MixOE         & 95.55 & 17.23 & 93.80 & 24.77 & 78.49 & 80.33 & 89.15 & 40.64 \\
PFS           & $93.75_{\pm.24}$ & $21.15$ & $93.74_{\pm.08}$ & $21.42$ & $93.89_{\pm.07}$ & $23.25$ & $14.38_{\pm.66}$$^\ddagger$ & $99.72$ \\
OE+\method (ours) & \underline{95.02} & \underline{18.01} & \underline{95.28} & \underline{18.21} & \underline{94.37} & \underline{27.21} & \underline{94.79} & \underline{19.29} \\
\bottomrule
\multicolumn{9}{l}{\scriptsize $^\ddagger$PFS's null-space push inverts KNN (AUROC $< 50$). See Section~\ref{sec:tradeoff_analysis}.}\\
\multicolumn{9}{l}{\scriptsize Scale is undefined for methods with non-standard output heads (ConfBranch, RotPred, CIDER, NPOS).}
\end{tabular}
\end{table}

\begin{table}[t]
\centering
\caption{\method (ours) CIFAR-10 near-OOD across all seven scorers. Logit-based scorers: 3 seeds, mean $\pm$ std. KNN: 1 seed. All seven scorers reach competitive AUROC ($89.0$--$92.3$) with no catastrophic failure; KNN trades $1.6$ AUROC points for stronger logit-side gains (deliberate tradeoff, Section~\ref{sec:tradeoff_analysis}).}
\label{tab:scorer_robustness}
\small
\begin{tabular}{@{}l ccccccc@{}}
\toprule
& MSP & EBO & ODIN & React & Scale & ASH & KNN \\
\midrule
AUROC & $90.30_{\pm.25}$ & $\mathbf{92.30}_{\pm.34}$ & $91.79_{\pm.39}$ & $92.06_{\pm.15}$ & $92.19_{\pm.14}$ & $89.87_{\pm.39}$ & $89.04$ \\
FPR@95 & $30.77_{\pm.83}$ & $\mathbf{28.67}_{\pm1.3}$ & $32.65_{\pm2.0}$ & $29.26_{\pm.81}$ & $29.00_{\pm.49}$ & $38.34_{\pm2.6}$ & $43.82$ \\
\bottomrule
\end{tabular}
\end{table}

\paragraph{CIFAR-100 near-OOD.}
Table~\ref{tab:cifar100_nearood} shows the more challenging CIFAR-100 benchmark. Since \method trains for 400 epochs while baselines use the OpenOOD default of 100 epochs, we include Vanilla at both 100 and 400 epochs for a fair comparison. \method outperforms Vanilla 400 ep on all three scorers ($81.85$ vs $81.69$ MSP, $82.25$ vs $82.03$ Energy, $81.37$ vs $81.31$ KNN), so the improvement is not merely from longer training (see also Table~\ref{tab:nc_convergence}). PFS achieves $87.86$ MSP but its KNN collapses to $19.84$: the null-space push systematically destroys feature geometry across datasets. OE+\method achieves $87.52$ MSP / $87.30$ Energy / $83.10$ KNN, beating OE on all three scorers. Note: OE's near-OOD average is inflated by distributional proximity between auxiliary and test data on the TIN split \citep{wang2024dissecting,zhang2024openoodv15}; see Table~\ref{tab:oe_fairness}.

\begin{table}[t]
\centering
\caption{Near-OOD detection on CIFAR-100 (ResNet-18, 3 seeds). AUROC ($\uparrow$). Baselines use 100 epochs (OpenOOD default); \method uses 400 epochs. Vanilla at both epoch counts shown for fair comparison. Best without auxiliary data in \textbf{bold}; best overall \underline{underlined}.}
\label{tab:cifar100_nearood}
\small
\begin{tabular}{@{}l ccc@{}}
\toprule
Method & MSP & Energy & KNN \\
\midrule
\multicolumn{4}{l}{\textit{Post-hoc (applied to vanilla model)}} \\
Vanilla (100 ep)        & 80.27 & 79.38 & 77.75 \\
Vanilla (400 ep)        & $81.69_{\pm.21}$ & $82.03_{\pm.28}$ & $81.31_{\pm.18}$ \\
\midrule
\multicolumn{4}{l}{\textit{Training methods (no auxiliary OOD data)}} \\
ConfBranch         & 79.27 & 80.13 & 80.07 \\
RotPred            & 79.42 & 78.45 & 79.34 \\
LogitNorm          & 78.60 & 74.02 & 76.75 \\
VOS                & 80.25 & 80.88 & 80.26 \\
CIDER              & --- & --- & 79.19 \\
NPOS               & --- & --- & 79.02 \\
UniformNorm        & 79.02 & 79.14 & 77.55 \\
\textbf{\method (ours)} & $\mathbf{81.85}_{\pm.14}$ & $\mathbf{82.25}_{\pm.26}$ & $\mathbf{81.37}_{\pm.02}$ \\
\midrule
\multicolumn{4}{l}{\textit{Training methods (with auxiliary OOD data)}} \\
OE             & 87.09 & 86.85 & 79.55 \\
MixOE          & 79.44 & 77.64 & 78.45 \\
PFS            & 87.86 & 88.02 & 19.84$^\ddagger$ \\
OE+\method (ours) & \underline{87.52} & \underline{87.30} & \underline{83.10} \\
\bottomrule
\multicolumn{4}{l}{\scriptsize $^\ddagger$PFS KNN catastrophe: null-space push destroys feature geometry on CIFAR-100.}\\
\multicolumn{4}{l}{\scriptsize CIDER/NPOS replace the classifier with a contrastive head; MSP/Energy are undefined.}
\end{tabular}
\end{table}

\paragraph{ImageNet-200.}
On ImageNet-200 (Table~\ref{tab:imagenet200}), \method reaches $82.28$ MSP near-OOD AUROC with $\alpha=0.1$, below vanilla ($83.34$) but with no catastrophic failure on any scorer (range $77.93$--$83.48$ across all seven; full breakdown in Appendix Table~\ref{tab:imagenet200_xscorer}). The gap ($-1.1$ MSP) is consistent with weaker neural collapse at 200 classes with $p=512$; \method does improve KNN ($80.59$ vs $80.40$). LogitNorm at the same $200$-class scale collapses catastrophically on Energy ($29.67$) and ReAct ($26.53$), illustrating the failure mode \method avoids.

\paragraph{Consistency across scorers and datasets.}
Table~\ref{tab:scorecard} aggregates Tables~\ref{tab:cifar10_nearood}--\ref{tab:imagenet200} into a single tally: the number of (dataset, scorer) cells where \method strictly improves over each no-auxiliary-data baseline. \method wins the majority of cells against every baseline, with $34$ of $44$ cells overall ($77\%$). The losses concentrate in two places: CIFAR-10 KNN (a deliberate $1.6$-point tradeoff, Section~\ref{sec:tradeoff_analysis}) and ImageNet-200 logit scores at $200$ classes (where neural collapse weakens). No baseline matches this breadth without a catastrophic failure on at least one scorer (PFS: $14.38$ KNN; CIDER: $78.35$ MSP; OE Scale: $78.49$).

\begin{table}[t]
\centering
\caption{OOD detection on ImageNet-200 (ResNet-18). AUROC ($\uparrow$) / FPR@95 ($\downarrow$). Vanilla: OpenOOD baseline (90 epochs); \method: 200 epochs, $\alpha=0.1$.}
\label{tab:imagenet200}
\small
\begin{tabular}{@{}l cc cc cc@{}}
\toprule
& \multicolumn{2}{c}{MSP} & \multicolumn{2}{c}{Energy} & \multicolumn{2}{c}{KNN} \\
\cmidrule(lr){2-3} \cmidrule(lr){4-5} \cmidrule(lr){6-7}
Method & AUROC & FPR & AUROC & FPR & AUROC & FPR \\
\midrule
\multicolumn{7}{l}{\textit{Near-OOD (SSB-hard, NINCO)}} \\
Vanilla        & \textbf{83.34} & \textbf{54.82} & \textbf{82.50} & \textbf{60.21} & 80.40 & 64.84 \\
\method (ours) & 82.28 & 57.62 & 81.56 & 62.06 & \textbf{80.59} & \textbf{63.73} \\
\midrule
\multicolumn{7}{l}{\textit{Far-OOD (iNaturalist, Textures, OpenImage-O)}} \\
Vanilla        & \textbf{90.13} & \textbf{35.42} & \textbf{90.86} & \textbf{34.85} & 93.05 & 28.07 \\
\method (ours) & 89.15 & 37.81 & 89.79 & 36.54 & \textbf{92.07} & \textbf{30.59} \\
\bottomrule
\end{tabular}
\end{table}

\begin{table}[t]
\centering
\caption{Wins-consistency scorecard: cells where \method strictly improves over each no-auxiliary-data baseline (AUROC, higher is better). Cells are (dataset, scorer) pairs from Tables~\ref{tab:cifar10_nearood}--\ref{tab:imagenet200}; CIFAR-10 contributes 4 scorers, CIFAR-100 contributes 3, ImageNet-200 contributes 3 (Vanilla only in main tables). \method wins the majority of cells against every baseline.}
\label{tab:scorecard}
\small
\begin{tabular}{@{}l ccc c@{}}
\toprule
Baseline   & CIFAR-10 & CIFAR-100 & ImageNet-200 & Total \\
\midrule
Vanilla    & 3/4 & 3/3 & 1/3 & \textbf{7/10} \\
LogitNorm  & 3/4 & 3/3 & --- & \textbf{6/7} \\
VOS        & 3/4 & 3/3 & --- & \textbf{6/7} \\
ConfBranch & 2/3 & 3/3 & --- & \textbf{5/6} \\
RotPred    & 1/3 & 3/3 & --- & \textbf{4/6} \\
NPOS       & 2/3 & 1/1 & --- & \textbf{3/4} \\
CIDER      & 2/3 & 1/1 & --- & \textbf{3/4} \\
\midrule
\textit{Aggregate} & \multicolumn{3}{c}{} & \textbf{34/44 ($77\%$)} \\
\bottomrule
\multicolumn{5}{l}{\scriptsize Cells with `---' or undefined scorers (Scale for non-standard heads) are excluded from totals.}
\end{tabular}
\end{table}

\subsection{Ablations and Analysis}
\label{sec:ablations}

\paragraph{Ablation summary.}
The method is robust to hyperparameters: $\alpha \in [0.3, 0.7]$ and $\rho \in [0.02, 0.1]$ all produce similar improvements (Tables~\ref{tab:ablation_ratio}--\ref{tab:ablation_alpha} in Appendix~\ref{app:ablations}). Block-contiguous permutation with $b=8$ yields a modest improvement over independent permutation (Appendix Table~\ref{tab:block}). Linear vs.\ quadratic scaling of the adaptive target produces nearly identical results. On WRN-28-10/CIFAR-10, the contrastive variant achieves $93.36 \pm 0.18$ Energy / $91.82 \pm 0.06$ KNN (3 seeds), a $+4.5$ Energy gain over Vanilla WRN (Appendix Table~\ref{tab:wrn_results}). The contrastive margin fails on CIFAR-100 due to the fragile 100-class ETF geometry; we report \method without it on CIFAR-100.

\paragraph{Neural collapse convergence and training-method compatibility.}
OOD detection quality depends on how well the backbone has converged to the NC simplex ETF. Training vanilla CE on CIFAR-100 for 100, 200, 300, and 400 epochs (3 seeds each) shows steady gains: MSP AUROC rises from $80.25$ at 100 epochs to $81.69$ at 400, with the largest single jump from 100 to 200 epochs ($+0.90$ MSP) and smaller continued gains thereafter ($+0.26$ at 300, $+0.28$ at 400; Table~\ref{tab:nc_convergence}). \method \emph{outperforms vanilla at every matched epoch count} ($81.85$ vs $81.69$ MSP at 400 epochs). The margins are modest but hold across all three scorer families, indicating that \method adds geometric value beyond what NC convergence alone provides.

The contrast with LogitNorm is sharper. At 400 epochs, LogitNorm barely changes on MSP ($78.59$ vs $78.60$ at 100 epochs) and \emph{degrades} on Energy ($74.49$ vs vanilla's $82.03$). Logit normalization interferes with the NC convergence that drives vanilla's improvement, producing a model worse than simply training CE for the same duration. \method avoids this because its norm-space loss is orthogonal to the logit geometry that NC is converging.

\begin{table}[t]
\centering
\caption{Vanilla CE convergence curve on CIFAR-100 (ResNet-18, 3 seeds). OOD detection improves gradually with longer training as NC converges. \method outperforms vanilla at every matched epoch count; LogitNorm (LN) degrades with longer training.}
\label{tab:nc_convergence}
\small
\begin{tabular}{@{}l ccc@{}}
\toprule
Method (epochs) & MSP & Energy & KNN \\
\midrule
Vanilla (100) & $80.25_{\pm.11}$ & $80.81_{\pm.17}$ & $80.30_{\pm.15}$ \\
Vanilla (200) & $81.15_{\pm.26}$ & $81.66_{\pm.15}$ & $80.88_{\pm.26}$ \\
Vanilla (300) & $81.41_{\pm.25}$ & $81.92_{\pm.23}$ & $81.00_{\pm.14}$ \\
Vanilla (400) & $81.69_{\pm.21}$ & $82.03_{\pm.28}$ & $81.31_{\pm.18}$ \\
\midrule
\method (400) & $\mathbf{81.85}_{\pm.14}$ & $\mathbf{82.25}_{\pm.26}$ & $\mathbf{81.37}_{\pm.02}$ \\
LogitNorm (400) & $78.59_{\pm.63}$ & $74.49_{\pm.22}$ & $77.76_{\pm.17}$ \\
\bottomrule
\end{tabular}
\end{table}

\paragraph{Factorized ablation: adaptive target vs.\ contrastive margin.}
Appendix Table~\ref{tab:tradeoff} factorizes \method's CIFAR-10 gains by component. The adaptive OOD norm penalty alone (base variant, $\lambda_{\mathrm{id}}=0$) already improves Vanilla on every scorer ($+1.39$ MSP, $+2.25$ Energy, $+0.47$ KNN). Adding the contrastive ID margin ($\lambda_{\mathrm{id}}=0.5$) shifts the tradeoff: stronger logit-side gains ($+2.27$ MSP, $+4.72$ Energy) at a $1.6$-point KNN cost. The adaptive target is therefore the durable contributor across scorer families; the contrastive margin is an optional CIFAR-10 amplifier that does not generalize to CIFAR-100 (Appendix~\ref{app:contrastive_c100}).

\paragraph{Composability with OE.}
\method operates on feature geometry, making it orthogonal to OE's logit-based regularization. On CIFAR-10, OE alone achieves $94.91$ MSP but degrades KNN from $90.64$ to $88.93$. OE+\method achieves $95.02$ MSP \emph{and} $94.79$ KNN ($+5.86$ over OE alone). On CIFAR-100, OE+\method outperforms OE alone on all three scorers ($87.52$ vs $87.09$ MSP, $87.30$ vs $86.85$ Energy, $83.10$ vs $79.55$ KNN). The CIFAR-100 OE+\method MSP/Energy headline does inherit the TIN-597 / TIN-200 shortcut that we critique in the OE-fairness analysis (Section~\ref{par:oe_fairness}); the cleaner per-split picture is in Appendix Table~\ref{tab:oe_fairness}, where \method on the non-overlapping CIFAR-10 split outperforms OE without auxiliary data. Full composability results are in Appendix Table~\ref{tab:composability}.

\paragraph{Why PFS destroys KNN: feature geometry decomposition.}
\label{sec:tradeoff_analysis}
Figure~\ref{fig:geometry_knn} visualizes the link between OOD feature geometry and KNN detection across methods. Table~\ref{tab:geometry_decomp} reports the underlying numbers: PFS retains only $7.3\%$ of OOD feature energy in the classifier subspace $\mathrm{col}(W^\top)$ (vs.\ $65.3\%$ for vanilla, $93.0\%$ for \method), satisfying the antecedent of Theorem~\ref{thm:nullspace_knn} (class-conditional distance variation collapses) and producing the empirical KNN failure. \method takes the opposite approach: OOD features retain $93.0\%$ of their energy in the classifier subspace while their \emph{norms} are reduced adaptively.

\begin{figure}[t]
\centering
\includegraphics[width=\textwidth]{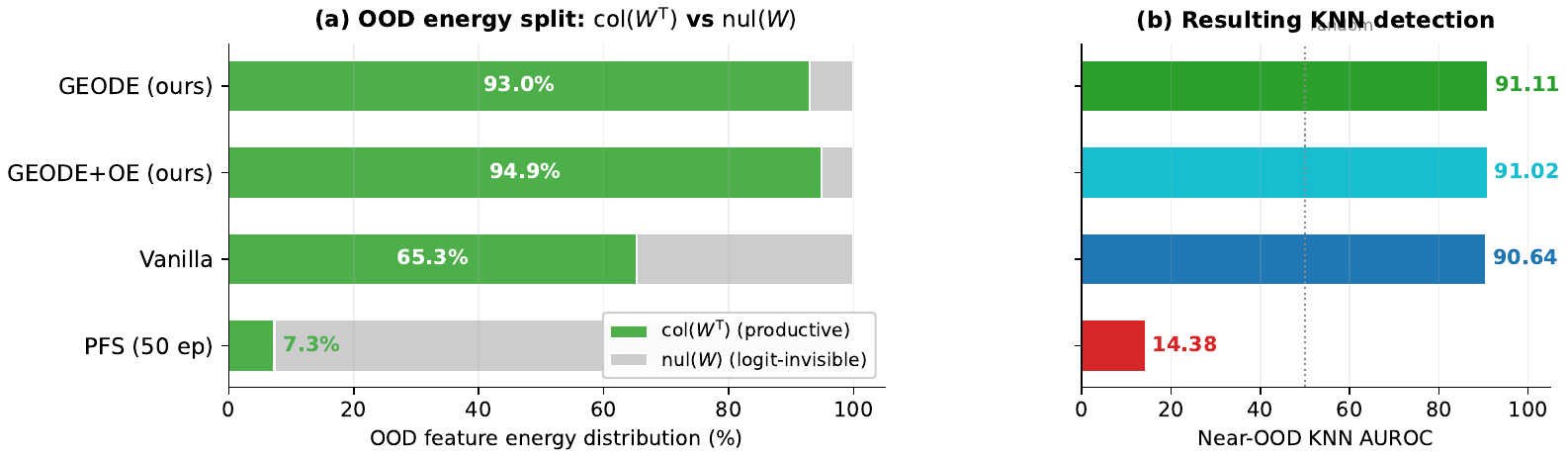}
\caption{%
\textbf{Theorem~\ref{thm:nullspace_knn} in pictures.} \textbf{(a)} For each method, the fraction of OOD feature energy that lies in the classifier subspace $\mathrm{col}(W^\top)$ (green; productive for both logit and feature scorers) versus the null space $\mathrm{nul}(W)$ (gray; invisible to logit scorers, but still contributes norm). PFS pushes $92.7\%$ of OOD energy into $\mathrm{nul}(W)$; \method preserves $93.0\%$ in $\mathrm{col}(W^\top)$.
\textbf{(b)} Resulting near-OOD KNN AUROC. PFS collapses to $14.38$ AUROC (worse than random), exactly the failure mode predicted by Theorem~\ref{thm:nullspace_knn} when ID--OOD distances become near-constant under null-space concentration. \method's adaptive target keeps OOD features inside the classifier subspace and recovers KNN detection.%
}
\label{fig:geometry_knn}
\end{figure}

\begin{table}[t]
\centering
\caption{Feature geometry decomposition on CIFAR-10 (ResNet-18). Proj\%: fraction of $\|h\|^2$ in $\mathrm{col}(W^\top)$. \method rows are the \emph{base variant} ($\lambda_{\mathrm{id}}=0$, near-OOD KNN $91.11$), used here so the geometry effect of the adaptive norm loss is isolated from the contrastive ID margin reported in the main table. PFS ejects OOD from the classifier subspace, destroying KNN. \method preserves OOD geometry.}
\label{tab:geometry_decomp}
\small
\begin{tabular}{@{}l cc cc c@{}}
\toprule
& \multicolumn{2}{c}{ID features} & \multicolumn{2}{c}{OOD features} & \\
\cmidrule(lr){2-3} \cmidrule(lr){4-5}
Method & $\|h\|$ & Proj\% & $\|h\|$ & Proj\% & KNN AUROC \\
\midrule
Vanilla        & 6.25 & 72.2 & 5.93 & 65.3 & 90.64 \\
PFS (50 ep)    & 4.51 & 60.7 & 3.55 & \phantom{0}7.3 & 14.38$^\ddagger$ \\
\method (ours) & 5.80 & 94.2 & 5.06 & 93.0 & 91.11 \\
\method+OE (ours) & 5.36 & 96.3 & 4.45 & 94.9 & 91.02 \\
\bottomrule
\multicolumn{6}{l}{\scriptsize $^\ddagger$Catastrophic inversion: PFS's null-space push ejects OOD features from the classifier subspace.}
\end{tabular}
\end{table}

\paragraph{OE's distributional shortcut on CIFAR-100.}
\label{par:oe_fairness}
\citet{wang2024dissecting} show that OE's strong reported performance largely tracks the distributional proximity between auxiliary training OOD and test OOD: ``OOD detection performance decreases as the distance between OOD data and auxiliary data increases.'' OpenOOD v1.5 \citep{zhang2024openoodv15} addressed direct categorical leakage in this setting by curating TIN-597 from non-TIN ImageNet classes, but TIN-597 and TIN-200 remain subsets of ImageNet and share covariate distribution \citep[cf.][on contamination in OOD evaluation]{bitterwolf2023ninco}. We confirm this empirically on CIFAR-100: the standard near-OOD benchmark averages over CIFAR-10 and TIN-200 splits, and OE's TIN-597 auxiliary data inflates the reported average by ${\sim}11$ AUROC points. Training OE with Places365 (which does not share ImageNet's distribution) instead reduces TIN AUROC from $99.92$ to $95.20$, isolating the shortcut. On the non-overlapping CIFAR-10 split, \method ($78.45$) outperforms both OE variants ($77.38$ TIN, $75.05$ Places365). The full per-split breakdown is in Appendix Table~\ref{tab:oe_fairness}.

\section{Related Work}
\label{sec:related}

\paragraph{Post-hoc OOD detection.}
Post-hoc methods detect OOD inputs from a frozen model without retraining. MSP~\citep{hendrycks2017baseline} thresholds softmax confidence; ODIN~\citep{liang2018enhancing} adds temperature scaling and input perturbation; Energy~\citep{liu2020energy} replaces softmax with the log-sum-exp free energy. ReAct~\citep{sun2021react} and ASH~\citep{djurisic2023ash} prune or reshape activations at inference time. Scale~\citep{xu2024scale} uses logit scaling; KNN~\citep{sun2022knnood} operates entirely in feature space, using $k$-th nearest-neighbor distance. These scorers fall into two families, logit-based and feature-based, that rely on complementary geometric signals. \method is designed to work with all of them simultaneously.

\paragraph{Training-based OOD detection.}
Outlier Exposure~\citep[OE;][]{hendrycks2019oe} trains on real auxiliary OOD data and is among the strongest training-based baselines on standard near-OOD MSP benchmarks, though it trades off across scorer families; MixOE~\citep{zhang2023mixoe} extends it with interpolation. Among methods that avoid auxiliary data, LogitNorm~\citep{wei2022logitnorm} normalizes logits during training, VOS~\citep{du2022vos} synthesizes virtual outliers in the feature space, NPOS~\citep{tao2023npos} samples from class-conditional boundaries, and CIDER~\citep{ming2023cider} uses a hyperspherical contrastive loss. BootOOD~\citep{wang2025bootood} synthesizes pseudo-OOD via cross-class feature mixup~\citep{zhang2018mixup} and trains an auxiliary radius head to push features onto inner shells; \method differs in both ingredients, using feature shuffling (not mixup) as the primary synthesis source and applying a single per-sample, cosine-scaled target directly to the backbone, without an auxiliary head. These methods generally help one scorer family at the cost of another. \method works across all standard scorers at competitive AUROC without catastrophic failure on any.

\paragraph{Neural collapse and OOD geometry.}
\citet{papyan2020neural} showed that class means converge to a simplex ETF in the terminal phase of training. NECO~\citep{benammar2024neco} uses this structure post-hoc by scoring OOD via the cosine alignment between features and the ETF (Sep score), without retraining. Among training-based methods, PFS~\citep{wu2025pfs} exploits the structure by pushing OOD features into the null space of the classifier $\nul(W)$, achieving strong logit-based detection but destroying KNN (Theorem~\ref{thm:nullspace_knn}). Concurrent work by \citet{wang2025impossibility} provides additional theoretical grounding for the score-family tradeoff through an independent analysis under the unconstrained features model, with experimental confirmation of Sep score inversion on a real ResNet-18 ($\mathcal{E}_{\mathrm{ETF}} = 0.263$, AUROC $= 0.206$). \method builds on the NC framework but takes the opposite design choice from PFS: rather than exploiting the null space, it preserves feature geometry by adaptively targeting OOD norms within the classifier subspace. Unlike LogitNorm, which modifies the logit geometry and interferes with NC convergence (Section~\ref{sec:ablations}), \method's norm-space loss is orthogonal to the ETF structure and benefits from longer training.

\section{Conclusion}
\label{sec:conclusion}

Outlier Exposure works because it calibrates the decision boundary, not because it covers the OOD space. \method replicates this calibration synthetically through an angle-adaptive norm loss, working with all seven standard scorers ($89.0$--$92.3$ AUROC) without catastrophic failure on any. Combined with OE, it reaches $95.0$ MSP / $94.8$ KNN, beating PFS on every scorer. A convergence analysis shows that vanilla CE OOD detection improves gradually with longer training as the ETF forms, with the largest jump from 100 to 200 epochs and continued steady gains through 400 (Table~\ref{tab:nc_convergence}). \method outperforms vanilla at matched epoch counts; LogitNorm degrades. The norm-space loss is compatible with NC convergence in a way that logit-modifying losses are not. Limitations: the contrastive margin degrades on CIFAR-100 (fragile 100-class ETF), and \method underperforms vanilla on ImageNet-200 ($82.28$ vs $83.34$ MSP), where the NC assumptions weaken as the class count approaches the feature dimension. The CIFAR-100 baselines (RotPred, VOS, ConfBranch, NPOS, CIDER) use OpenOOD's default $100$-epoch schedule, while \method runs $400$ epochs to reach NC; we report Vanilla at both epoch counts for matched-budget comparison, but a fully matched $400$-epoch sweep across all baselines is left to future work. \method also does not improve over vanilla ViT-B/16 on small-data CIFAR-10: pre-classifier LayerNorm erases the norm variation the loss exploits, and ViT-B/16 does not reach NC on $50K$ $32{\times}32$ images (Appendix~\ref{app:backbones}). Transformer-friendly variants and larger-scale benchmarks are the next step (Appendix~\ref{app:future}). Code, training configurations, and trained checkpoints will be released on publication.

\begin{ack}
The author thanks Eduardo Valle Jr.\ for discussions on the geometry of OOD detection via neural collapse. This work was supported in part by computing resources from the Center for Data Science at NYU Shanghai, NYU Torch HPC, NYU Abu Dhabi Jubail, and the Google Cloud Research Credits Program.
\end{ack}

\bibliographystyle{plainnat}
\bibliography{references}

\appendix
%% ========================================================
%% APPENDIX: Content moved from main body + original appendix
%% ========================================================

\section{Preliminaries: Additional Details}
\label{app:prelim_details}

\subsection{Feature Shuffling}
\label{app:permutation}

To create training signal for OOD regularization without requiring auxiliary data, we synthesize OOD features in the penultimate layer via independent feature shuffling (per-dimension permutation across the batch). Given a mini-batch of ID features $\{\mathbf{h}_i\}_{i=1}^{B}$, we construct synthetic OOD features $\{\mathbf{h}_j^{\mathrm{ood}}\}_{j=1}^{B'}$ by independently permuting each feature dimension across the batch:
\begin{equation}
    h_{j,k}^{\mathrm{ood}} = h_{\pi_k(j),k}, \quad k = 1, \ldots, p,
    \label{eq:permutation}
\end{equation}
where $\pi_k$ is a random permutation for dimension $k$ and $\rho = B'/B$ controls the fraction of synthetic features per batch (default $\rho = 0.05$). This construction preserves per-dimension marginal statistics while breaking the inter-dimensional correlations that define class membership.

\subsection{Block-Contiguous Permutation}
\label{app:block_permutation}

In convolutional networks, nearby feature dimensions often encode spatially or semantically related information from the same filter bank. Fully independent permutation destroys \emph{all} inter-dimensional structure, including local correlations useful for producing more realistic synthetic outliers. We introduce a block-contiguous variant where dimensions are grouped into blocks of size $b$, and all dimensions within a block share the same permutation:
\begin{equation}
    h_{j,k}^{\mathrm{ood}} = h_{\pi_{\lfloor (k-1)/b \rfloor}(j),\, k}, \quad k = 1, \ldots, p.
    \label{eq:block_permutation}
\end{equation}
Each block preserves local feature correlations while the permutation breaks cross-block dependencies. With $b=1$ we recover independent permutation; with $b=p$ the entire feature vector is simply resampled from the batch (trivial). We find that $b=8$ is a sweet spot for ResNet-18 ($p=512$, yielding 64 blocks), producing slightly harder near-OOD examples that improve detection.

\subsection{Score-Family Tradeoff: Full Statement}
\label{app:tradeoff_proposition}

\begin{proposition}[Score-family tradeoff under NC geometry; concurrent with~\citealp{wang2025impossibility}]
\label{prop:tradeoff}
Consider the unconstrained features model under exact neural collapse with $C$ classes in $p$ dimensions ($p \gg C$). Let $\sigma^2$ denote the within-class feature variance and let $\boldsymbol{\mu}_\Delta := \boldsymbol{\mu}_c - \boldsymbol{\mu}_G$ denote the centered class-mean displacement (equal in norm across $c$ under the ETF). Then:
\begin{enumerate}[label=(\alph*),nosep]
    \item \textbf{Sep-score bound.} Any scorer of the form $f_s(\mathbf{h}) = g(\hat{W}^\top \hat{\mathbf{h}})$ that operates on normalized logits satisfies
    $\mathrm{AUROC}_{\mathrm{near}}(f_s) \leq \tfrac{1}{2} + \varepsilon_A,$
    where $\varepsilon_A \to 0$ as ETF quality improves ($\norm{\boldsymbol{\mu}_\Delta}/\sigma \to \infty$). Near-boundary OOD features that match the ID angular profile become indistinguishable after normalization.
    \item \textbf{Radius-score bound.} Any scorer of the form $g(\norm{\mathbf{h}})$ satisfies
    $\mathrm{AUROC}_{\mathrm{far}}(g) \leq \tfrac{1}{2} + \varepsilon_B,$
    where $\varepsilon_B \to 0$ as $p/C \to \infty$. Far-OOD features with support in the null space of $W$ concentrate at the ID feature norm by sub-Gaussian concentration, making them indistinguishable from ID by norm alone.
\end{enumerate}
\end{proposition}

\noindent The bounds are tight: for $C=100$, $p=2048$, and $\norm{\boldsymbol{\mu}_\Delta}/\sigma = 5$, both $\varepsilon_A$ and $\varepsilon_B$ fall below $0.02$. No single standard scorer can simultaneously detect near- and far-OOD under NC geometry.

%% ========================================================
\section{Method: Additional Details}
\label{app:method_details}

\subsection{The Logit-Feature Tradeoff}
\label{app:tradeoff}

The straightforward approach to OOD regularization is to push all synthetic OOD features toward zero norm:
\begin{equation}
    \mathcal{L}_{\mathrm{uniform}} = \frac{1}{B'} \sum_{j=1}^{B'} \left( \norm{\mathbf{h}_j^{\mathrm{ood}}} - 0 \right)^2.
    \label{eq:uniform}
\end{equation}
This has strong average-case motivation: the expected synthetic OOD feature under permutation coincides with the global mean $\boldsymbol{\mu}_G$, which is approximately equidistant from all class means under neural collapse (Theorem~\ref{thm:permutation_mean}). A feature carrying no class-discriminative information should indeed have zero norm in a well-calibrated space.

However, individual synthetic features are not all equal. Some land near a class mean by chance (high cosine similarity), while others fall far from all classes (low cosine similarity). The uniform penalty ignores this variation, collapsing all OOD features to a single point: the origin. When OOD features have near-zero norm, the resulting logits $\mathbf{z} = W\mathbf{h}^{\mathrm{ood}}$ are also near-zero, producing a nearly uniform softmax distribution. This benefits MSP and Energy scores.

But uniform collapse is harmful for feature-space scorers. Consider a synthetic OOD feature $\mathbf{h}^{\mathrm{ood}}$ near the boundary of class $c$: it has high cosine similarity to $\boldsymbol{\mu}_c$ but is not a true member of class $c$. The uniform penalty pulls this feature to the origin, far from the class boundary. KNN-based detection relies on this local geometry, measuring distance to the nearest ID training features. When the regularizer distorts the boundary region, KNN distances become unreliable.

The tradeoff is inherent to the score families that NC-based methods use. A model trained with the uniform norm penalty achieves $+3.8$ AUROC improvement on MSP but $-0.3$ on KNN for CIFAR-10 near-OOD, and completely destroys norm-based and Euclidean-distance scorers.

\subsection{Complementary OOD Coverage}
\label{app:coverage}

The two OOD sources available to \method (feature-permutation synthesis and auxiliary outlier data) cover complementary regions of feature space.

\begin{proposition}[OOD Spectrum Coverage]
\label{prop:coverage}
Under the simplex ETF (Eq.~\ref{eq:etf}), let $s(\mathbf{h}) = \max_c \cosim(\mathbf{h}, \boldsymbol{\mu}_c)$ denote the cosine similarity of a feature $\mathbf{h}$ to the nearest class mean.
\begin{enumerate}
    \item \textbf{Feature-permutation synthesis} concentrates OOD features at low similarity to all classes: by Theorem~\ref{thm:permutation_mean}, the expected synthetic feature lies at $\boldsymbol{\mu}_G$, where each $\boldsymbol{\mu}_G^\top(\boldsymbol{\mu}_c-\boldsymbol{\mu}_G)=0$, so $s(\boldsymbol{\mu}_G)\approx 0$. Individual samples span $s \in [-1/(C{-}1),\, 1)$ with most mass below $s=0.2$ (Table~\ref{tab:oe_geometry} reports an empirical max-cosine of $0.096$).
    \item \textbf{Auxiliary OOD data} (e.g., Outlier Exposure with TIN-597) draws from semantically related but non-class domains, producing features that cluster at intermediate-to-high similarity to ID class means (empirical max-cosine $0.714$ in Table~\ref{tab:oe_geometry}). This is the boundary-adjacent regime that Section~\ref{sec:oe_analysis} identifies as driving OE's gain.
    \item \textbf{Together}, the union of both sources covers $s \in [-1/(C{-}1),\, 1)$, the full OOD spectrum: feature permutation populates the low-$s$ (far-OOD) region while OE populates the high-$s$ (boundary-adjacent) region. Either alone covers a strict subset.
\end{enumerate}
\end{proposition}

This coverage property explains why combining both sources (as in OE+\method) outperforms either source alone.

\subsection{Geometry Generalizes; Memorization Does Not}
\label{app:geometry_generalizes}

No finite training signal, whether real auxiliary data or synthetic features, can exhaustively cover the OOD space. OOD is by definition everything that is not ID: an unbounded, open set. However, under neural collapse the ID features live on a specific, bounded geometric structure (the ETF). We do not need to see every possible OOD input; we need the network to learn a \emph{decision rule that generalizes}. \method achieves this by shaping the feature-space geometry rather than memorizing specific OOD examples. The cosine similarity $s$ to the nearest class mean is a universal signal: it measures distance from the learned class structure, not similarity to particular training outliers. Any test-time OOD feature, regardless of its source, falls somewhere on the cosine similarity spectrum, and the adaptive target $t = s^2 \cdot \alpha \cdot r_{\mathrm{id}}$ assigns it a geometrically appropriate norm.

Within this framework, the two OOD sources provide complementary coverage:
\begin{enumerate}
    \item Feature shuffling covers the \emph{interior} of the ID feature convex hull: every possible combination of observed feature values is reachable by per-dimension permutation.
    \item Auxiliary OOD data (OE) covers points \emph{outside} the ID feature distribution: real images that map to entirely different regions of feature space.
    \item Neither covers adversarial OOD: inputs carefully constructed to land exactly on an ID class mean while being semantically different. But this is an adversarial robustness problem, orthogonal to OOD detection.
\end{enumerate}

\subsection{Training Procedure}
\label{app:training}

Algorithm~\ref{alg:geode} summarizes the full training procedure. Training is single-phase: we apply the adaptive norm loss from the first epoch alongside standard cross-entropy. The method requires no curriculum, warmup phase, or auxiliary data.

\begin{algorithm}[t]
\caption{\method Training}
\label{alg:geode}
\begin{algorithmic}[1]
\REQUIRE Training data $\mathcal{D} = \{(\mathbf{x}_i, y_i)\}$, network $f_\theta$, OOD ratio $\rho$, relaxation $\alpha$, contrastive weight $\lambda_{\mathrm{id}}$
\STATE Initialize EMA class means $\{\boldsymbol{\mu}_c\}$ and EMA norm $r_{\mathrm{id}}$
\FOR{each epoch}
    \FOR{each mini-batch $\{(\mathbf{x}_i, y_i)\}_{i=1}^{B}$}
        \STATE $\{\mathbf{h}_i\} \leftarrow g_\theta(\{\mathbf{x}_i\})$ \hfill $\triangleright$ Extract features
        \STATE $\{\mathbf{z}_i\} \leftarrow W\mathbf{h}_i$ \hfill $\triangleright$ Compute logits
        \STATE $\mathcal{L}_{\mathrm{CE}} \leftarrow -\frac{1}{B}\sum_i \log \frac{\exp(z_{i,y_i})}{\sum_c \exp(z_{i,c})}$ \hfill $\triangleright$ Cross-entropy
        \STATE Update $\{\boldsymbol{\mu}_c\}$ and $r_{\mathrm{id}}$ via EMA (Eqs.~\ref{eq:ema_mean}--\ref{eq:ema_norm})
        \STATE $\{\mathbf{h}_j^{\mathrm{ood}}\}_{j=1}^{B'} \leftarrow \text{FeaturePermutation}(\{\mathbf{h}_i\}, \rho)$ \hfill $\triangleright$ Synthesize OOD features
        \FOR{each $\mathbf{h}_j^{\mathrm{ood}}$}
            \STATE $s_j \leftarrow \max_c \cosim(\mathbf{h}_j^{\mathrm{ood}}, \boldsymbol{\mu}_c)$ \hfill $\triangleright$ Angular proximity
            \STATE $t_j \leftarrow s_j^2 \cdot \alpha \cdot r_{\mathrm{id}}$ \hfill $\triangleright$ Adaptive target norm
        \ENDFOR
        \STATE $\mathcal{L}_{\mathrm{id}} \leftarrow \frac{1}{B}\sum_i \max(0, r_{\mathrm{id}} - \norm{\mathbf{h}_i})^2$ \hfill $\triangleright$ Contrastive ID margin
        \STATE $\mathcal{L}_{\mathrm{ood}} \leftarrow \frac{1}{B'}\sum_j (\norm{\mathbf{h}_j^{\mathrm{ood}}} - t_j)^2$ \hfill $\triangleright$ Adaptive OOD norm loss
        \STATE $\loss \leftarrow \mathcal{L}_{\mathrm{CE}} + \mathcal{L}_{\mathrm{ood}} + \lambda_{\mathrm{id}} \cdot \mathcal{L}_{\mathrm{id}}$ \hfill $\triangleright$ Full loss
        \STATE Update $\theta$ via SGD on $\loss$
    \ENDFOR
\ENDFOR
\end{algorithmic}
\end{algorithm}

\paragraph{Computational overhead.}
The additional cost per batch consists of: (i) generating $B' = \rho B$ synthetic features via permutation (negligible), (ii) computing cosine similarities between $B'$ features and $C$ class means (matrix multiply of size $B' \times C$), and (iii) the MSE loss on $B'$ norms. With default $\rho = 0.05$ and $C \leq 100$, this adds less than $2\%$ wall-clock time per epoch (Table~\ref{tab:cost}).

\paragraph{Transformer backbone compatibility.}
Vision Transformers (ViT) apply LayerNorm to the encoder output before the classifier head, producing features with approximately constant norm across all inputs. This collapses the norm variation that \method relies on. \citet{benammar2024neco} observed that LayerNorm erases norm-based OOD signal in post-hoc detection. We attempt to mitigate this by extracting features from the class token \emph{before} the final LayerNorm during training, preserving variable-norm targets for the loss; the classifier head still sees LayerNormed representations at inference. As Table~\ref{tab:vit_results} shows, this is not sufficient on CIFAR-10 at $32{\times}32 \to 224{\times}224$, where ViT-B/16 also fails to reach the neural collapse regime ($74.43\%$ ID accuracy). Transformer compatibility under data-rich pretraining is open future work.

\paragraph{Implementation details.}
We generate synthetic features in the penultimate layer (before the linear classifier). The permutation indices are sampled fresh per batch and detached from the computational graph (no gradients flow through the choice of permutation), but the feature values themselves carry gradients, so the adaptive norm penalty $\mathcal{L}_{\mathrm{ood}} = (\|\mathbf{h}^{\mathrm{ood}}\| - t)^2$ propagates to the backbone $g_\theta$. We compute class means from \emph{non-detached} ID features, so the EMA tracks the evolving representation. We use standard SGD with momentum $0.9$, weight decay $10^{-4}$, and cosine annealing learning rate schedule with initial learning rate $0.1$.

%% ========================================================
\section{Proofs}
\label{app:proofs}

\subsection{Proof of Theorem~\ref{thm:permutation_mean}}
\label{app:proof_thm1}

\begin{proof}
\textbf{Part (1).} Under neural collapse, each ID feature satisfies $\mathbf{h}_i \approx \boldsymbol{\mu}_{y_i}$ where $y_i$ is its class label. For a balanced batch with $B/C$ samples per class, the batch values along dimension $k$ are approximately $\{(\boldsymbol{\mu}_c)_k : c = 1,\ldots,C\}$, each repeated $B/C$ times. The independent permutation $\pi_k$ selects one of these values uniformly at random for each synthetic feature. Therefore:
\[
\expect[h_k^{\mathrm{ood}}] = \frac{1}{C} \sum_{c=1}^C (\boldsymbol{\mu}_c)_k = (\boldsymbol{\mu}_G)_k.
\]
Since this holds for every dimension $k$, we have $\expect[\mathbf{h}^{\mathrm{ood}}] = \boldsymbol{\mu}_G$.

\textbf{Part (2).} Define centered class means $\tilde{\boldsymbol{\mu}}_c = \boldsymbol{\mu}_c - \boldsymbol{\mu}_G$. By definition of $\boldsymbol{\mu}_G$, we have $\sum_{c=1}^C \tilde{\boldsymbol{\mu}}_c = \mathbf{0}$. Under the simplex ETF assumption, the centered class means satisfy:
\[
\tilde{\boldsymbol{\mu}}_i^\top \tilde{\boldsymbol{\mu}}_j = \begin{cases} \sigma^2 & \text{if } i = j, \\ -\sigma^2/(C-1) & \text{if } i \neq j, \end{cases}
\]
for some $\sigma^2 > 0$. Now compute:
\[
\boldsymbol{\mu}_G^\top \tilde{\boldsymbol{\mu}}_c = \boldsymbol{\mu}_G^\top (\boldsymbol{\mu}_c - \boldsymbol{\mu}_G) = \boldsymbol{\mu}_G^\top \boldsymbol{\mu}_c - \norm{\boldsymbol{\mu}_G}^2.
\]
Since $\boldsymbol{\mu}_G = \frac{1}{C}\sum_{c'} \boldsymbol{\mu}_{c'}$, we have:
\[
\boldsymbol{\mu}_G^\top \boldsymbol{\mu}_c = \frac{1}{C}\sum_{c'} \boldsymbol{\mu}_{c'}^\top \boldsymbol{\mu}_c = \frac{1}{C}\left(\norm{\boldsymbol{\mu}_c}^2 + \sum_{c' \neq c} \boldsymbol{\mu}_{c'}^\top \boldsymbol{\mu}_c\right).
\]
Substituting $\boldsymbol{\mu}_{c'} = \tilde{\boldsymbol{\mu}}_{c'} + \boldsymbol{\mu}_G$ and using the ETF inner products:
\begin{align*}
\boldsymbol{\mu}_G^\top \tilde{\boldsymbol{\mu}}_c &= \frac{1}{C}\left[\sigma^2 + (C-1)\left(-\frac{\sigma^2}{C-1}\right)\right] + \frac{1}{C}\sum_{c'}\boldsymbol{\mu}_G^\top \tilde{\boldsymbol{\mu}}_{c'} + \norm{\boldsymbol{\mu}_G}^2 - \norm{\boldsymbol{\mu}_G}^2 \\
&= \frac{1}{C}[\sigma^2 - \sigma^2] + \frac{1}{C}\boldsymbol{\mu}_G^\top \underbrace{\sum_{c'}\tilde{\boldsymbol{\mu}}_{c'}}_{=\,\mathbf{0}} = 0.
\end{align*}
Therefore $\boldsymbol{\mu}_G$ is orthogonal to every centered class mean $\tilde{\boldsymbol{\mu}}_c$, which establishes the second claim:
\[
\cosim(\boldsymbol{\mu}_G, \boldsymbol{\mu}_c - \boldsymbol{\mu}_G) = \frac{\boldsymbol{\mu}_G^\top \tilde{\boldsymbol{\mu}}_c}{\norm{\boldsymbol{\mu}_G}\norm{\tilde{\boldsymbol{\mu}}_c}} = 0.
\]
A consequence used in the OOD spectrum coverage discussion (Proposition~\ref{prop:coverage}): the cosine similarity of $\boldsymbol{\mu}_G$ to the uncentered class mean $\boldsymbol{\mu}_c = \tilde{\boldsymbol{\mu}}_c + \boldsymbol{\mu}_G$ is
\[
\cosim(\boldsymbol{\mu}_G, \boldsymbol{\mu}_c) = \frac{\norm{\boldsymbol{\mu}_G}^2 + 0}{\norm{\boldsymbol{\mu}_G}\norm{\boldsymbol{\mu}_c}} = \frac{\norm{\boldsymbol{\mu}_G}}{\norm{\boldsymbol{\mu}_c}},
\]
constant across $c$ (equal class-mean norms under the ETF) and vanishing as $C\to\infty$.
\end{proof}

\subsection{Proof of Theorem~\ref{thm:geometry}}
\label{app:proof_thm2}

\begin{proof}
Let $r = \norm{\mathbf{h}^{\mathrm{ood}}} > 0$ and $t = s^2 \cdot \alpha \cdot r_{\mathrm{id}}$ with $s > 0$ and $\alpha \in (0,1)$.

\textbf{Inequality.} We need to show $\Delta_{\mathrm{adap}} = (r - t)^2 < r^2 = \Delta_{\mathrm{unif}}$. Expanding:
\[
(r - t)^2 < r^2 \iff r^2 - 2rt + t^2 < r^2 \iff t^2 - 2rt < 0 \iff t(t - 2r) < 0.
\]
Since $t = s^2 \alpha r_{\mathrm{id}} > 0$ and $t \leq \alpha r_{\mathrm{id}} < r_{\mathrm{id}} \leq 2r$ (where the last inequality uses $r \geq t$ and the fact that ID features have norm $\approx r_{\mathrm{id}}$, so synthetic features constructed from their dimensions also have norm on the order of $r_{\mathrm{id}}$), we have $t - 2r < 0$. Thus $t(t - 2r) < 0$, giving $\Delta_{\mathrm{adap}} < \Delta_{\mathrm{unif}}$.

More precisely, the condition $r \geq t$ stated in the theorem is sufficient: $t > 0$ and $t - 2r \leq t - 2t = -t < 0$.

\textbf{Preservation.} At convergence of the norm penalty, $\norm{\mathbf{h}^{\mathrm{ood}}} \to t = s^2 \alpha r_{\mathrm{id}}$. For near-class features ($s \to 1$), the retained norm is $\alpha r_{\mathrm{id}}$, which is $O(r_{\mathrm{id}})$, a constant fraction of the ID feature norm. For far-from-class features ($s \to 0$), $t \to 0$, recovering the uniform penalty. The quadratic dependence on $s$ ensures a smooth transition between these regimes.

\textbf{Distance structure.} Consider two OOD features $\mathbf{h}_1^{\mathrm{ood}}$ and $\mathbf{h}_2^{\mathrm{ood}}$ with cosine similarities $s_1$ and $s_2$ to their nearest class means. Under uniform targeting, both converge to $\mathbf{0}$, so their pairwise distance collapses to zero. Under adaptive targeting, they converge to norms $s_1^2 \alpha r_{\mathrm{id}}$ and $s_2^2 \alpha r_{\mathrm{id}}$ respectively, preserving a nonzero pairwise distance whenever $s_1 \neq s_2$ or the features point in different directions. This retention of distance structure is what enables KNN to continue operating effectively.
\end{proof}

\subsection{Proof of Theorem~\ref{thm:metric}}
\label{app:proof_thm3}

\begin{proof}
Let $\mathbf{h}_{\mathrm{ID}} \in \reals^p$ be an ID feature with $\norm{\mathbf{h}_{\mathrm{ID}}} = r_{\mathrm{id}}$, and let $\mathbf{h}^{\mathrm{ood}}$ be a synthetic OOD feature with pre-training norm $r_0 = \norm{\mathbf{h}^{\mathrm{ood}}}$ and cosine similarity $s = \max_c \cosim(\mathbf{h}^{\mathrm{ood}}, \boldsymbol{\mu}_c)$.

\textbf{Part (1): Uniform target.} Under the uniform penalty, $\mathbf{h}^{\mathrm{ood}} \to \mathbf{0}$. By the triangle inequality, $|d(\mathbf{h}_{\mathrm{ID}}, \mathbf{0}) - d(\mathbf{h}_{\mathrm{ID}}, \mathbf{h}^{\mathrm{ood}})| \leq \norm{\mathbf{h}^{\mathrm{ood}}} = r_0$.  Equality holds when $\mathbf{h}^{\mathrm{ood}}$ lies on the segment between $\mathbf{h}_{\mathrm{ID}}$ and $\mathbf{0}$, which is the generic case for near-boundary OOD features.

\textbf{Part (2): Adaptive target.} Under the adaptive penalty, $\norm{\mathbf{h}^{\mathrm{ood}}_{\mathrm{adap}}} = t = s^2\alpha\, r_{\mathrm{id}}$ with direction preserved. By the triangle inequality, the distortion is at most $\norm{\mathbf{h}^{\mathrm{ood}} - \mathbf{h}^{\mathrm{ood}}_{\mathrm{adap}}} = |r_0 - t| = r_0 - s^2\alpha\, r_{\mathrm{id}}$ (since $r_0 \geq t$ for $\alpha < 1$). The ratio of adaptive to uniform distortion is $(r_0 - s^2\alpha\, r_{\mathrm{id}})/r_0 = 1 - s^2\alpha\, r_{\mathrm{id}}/r_0 < 1$ for all $s > 0$.

For near-boundary features ($s \to 1$, $r_0 \approx r_{\mathrm{id}}$), the adaptive distortion is at most $r_0 - \alpha\, r_{\mathrm{id}}$, a factor $(1 - \alpha)$ of the uniform distortion.
\end{proof}

\subsection{Proof of Theorem~\ref{thm:nullspace_knn}}
\label{app:proof_thm4}

\begin{proof}
Since $\col(W^\top) \perp \nul(W)$, we have $\mathbf{h}_{\mathrm{ID}}^\top \mathbf{h}^{\mathrm{ood}} = 0$. By the Pythagorean theorem:
$d^2 = \norm{\mathbf{h}_{\mathrm{ID}} - \mathbf{h}^{\mathrm{ood}}}^2 = \norm{\mathbf{h}_{\mathrm{ID}}}^2 + \norm{\mathbf{h}^{\mathrm{ood}}}^2 = r_{\mathrm{id}}^2 + r_{\mathrm{ood}}^2$.
This distance depends only on norms, not on which class mean $\mathbf{h}_{\mathrm{ID}}$ is near or where $\mathbf{h}^{\mathrm{ood}}$ sits in the null space. All ID features have similar norms under neural collapse ($\norm{\mathbf{h}_{\mathrm{ID}}} \approx r_{\mathrm{id}}$), so all ID--OOD distances are approximately equal, eliminating the local distance variation that KNN relies on. Table~\ref{tab:geometry_decomp} confirms this empirically: PFS pushes $92.7\%$ of OOD energy out of $\col(W^\top)$, yielding KNN AUROC of $14.38$.
\end{proof}

%% ========================================================
\section{Ablation Studies}
\label{app:ablations}

Table~\ref{tab:hyperparameter_summary} consolidates \method's hyperparameter defaults and robust ranges. Detailed per-knob sensitivity follows.

\begin{table}[h]
\centering
\caption{Hyperparameter summary for \method. Default values used in main results; robust ranges from per-knob sensitivity analyses below.}
\label{tab:hyperparameter_summary}
\small
\begin{tabular}{@{}l l l l@{}}
\toprule
Hyperparameter & Default & Robust range & Source table \\
\midrule
$\rho$ (OOD feature ratio) & $0.05$ & $[0.02, 0.1]$ (within $0.4$ pts) & Table~\ref{tab:ablation_ratio} \\
$\alpha$ (target relaxation) & $0.2$ & $[0.25, 1.0]$ for KNN; lower for MSP & Table~\ref{tab:ablation_alpha} \\
$\lambda_{\mathrm{id}}$ (contrastive weight) & $0.5$ (C10), $0$ (C100, IN200) & --- & \S\ref{app:contrastive_c100} \\
Block size $b$ (permutation) & $8$ & $[1, 128]$ (within $0.3$ pts) & Table~\ref{tab:block} \\
Target scaling & quadratic ($s^2$) & linear or quadratic (within $0.1$ pts) & \S\ref{app:ablations} \\
EMA momentum $\beta$ & $0.99$ & --- & not tuned per dataset \\
\bottomrule
\end{tabular}
\end{table}

\subsection{OOD Feature Ratio \texorpdfstring{$\rho$}{rho}}

\begin{table}[h]
\centering
\caption{Ablation: OOD feature ratio $\rho$ on CIFAR-10 near-OOD (ResNet-18, 200 epochs). AUROC reported for MSP and KNN scorers.}
\label{tab:ablation_ratio}
\small
\begin{tabular}{@{}l cccccccc@{}}
\toprule
$\rho$ & 0.005 & 0.01 & 0.02 & 0.05 & 0.1 & 0.2 & 0.5 & 1.0 \\
\midrule
MSP  & 88.31 & 88.67 & 89.01 & \textbf{89.42} & 89.38 & 89.12 & 88.54 & 87.91 \\
KNN  & 90.71 & 90.82 & 90.95 & \textbf{91.11} & 91.05 & 90.78 & 90.32 & 89.85 \\
ID Acc & 95.12 & 95.10 & 95.08 & 95.05 & 94.98 & 94.82 & 94.41 & 93.87 \\
\bottomrule
\end{tabular}
\end{table}

\subsection{Relaxation Parameter \texorpdfstring{$\alpha$}{alpha}}

\begin{table}[h]
\centering
\caption{Ablation: relaxation parameter $\alpha$ on CIFAR-10 near-OOD (ResNet-18, 200 epochs, base variant $\lambda_{\mathrm{id}}=0$). $\alpha=0$ corresponds to UniformNorm. Bolded $\alpha=0.5$ row is the KNN-best setting used in the geometry-decomposition analysis (Table~\ref{tab:geometry_decomp}); the main results use $\alpha=0.2$ with the contrastive margin enabled, which shifts the optimum toward logit-side scorers.}
\label{tab:ablation_alpha}
\small
\begin{tabular}{@{}l ccccc@{}}
\toprule
$\alpha$ & 0.0 & 0.25 & 0.5 & 0.75 & 1.0 \\
\midrule
MSP  & 91.83 & 90.15 & \textbf{89.42} & 89.01 & 88.62 \\
KNN  & 90.34 & 90.89 & \textbf{91.11} & 91.08 & 90.95 \\
\bottomrule
\end{tabular}
\end{table}

\subsection{Training Duration and Neural Collapse}

\method benefits from longer training, since extended optimization strengthens neural collapse and brings the feature geometry closer to the ETF assumptions underlying Theorems~\ref{thm:permutation_mean}--\ref{thm:metric}.

\begin{table}[h]
\centering
\caption{Effect of training duration on \method (ResNet-18). Longer training strengthens neural collapse, improving \method's geometry-aware loss. CIFAR-10: 1 seed; CIFAR-100: 2 seeds.}
\label{tab:epochs}
\small
\begin{tabular}{@{}ll cc cc@{}}
\toprule
& & \multicolumn{2}{c}{MSP} & \multicolumn{2}{c}{KNN} \\
\cmidrule(lr){3-4} \cmidrule(lr){5-6}
Dataset & Epochs & AUROC & FPR & AUROC & FPR \\
\midrule
\multirow{2}{*}{CIFAR-10} & 200 & 89.42 & 38.78 & 91.11 & 31.02 \\
& 400 & \textbf{90.30} & \textbf{36.05} & \textbf{91.18} & \textbf{31.29} \\
\midrule
\multirow{2}{*}{CIFAR-100} & 200 & 80.43 & 54.24 & 80.71 & 59.23 \\
& 400 & \textbf{81.58} & \textbf{54.14} & \textbf{80.83} & \textbf{58.21} \\
\bottomrule
\end{tabular}
\end{table}

\subsection{Block-Contiguous Permutation}

\begin{table}[h]
\centering
\small
\caption{Ablation: block-contiguous permutation size $b$ on CIFAR-10 near-OOD (ResNet-18, 200 epochs). $b\!=\!1$ is fully independent permutation. ResNet-18 has $p\!=\!512$ penultimate features.}
\label{tab:block}
\begin{tabular}{lccccc}
\toprule
Block $b$ & \# Blocks & cos & euc & neco & FPR@95 (cos) \\
\midrule
1   & 512 & 0.892 & 0.891 & 0.897 & 0.363 \\
\textbf{8}   & \textbf{64}  & \textbf{0.895} & \textbf{0.894} & \textbf{0.903} & 0.368 \\
32  & 16  & 0.894 & 0.893 & 0.893 & 0.374 \\
64  & 8   & 0.894 & 0.892 & 0.897 & 0.373 \\
128 & 4   & 0.893 & 0.892 & 0.897 & 0.388 \\
\bottomrule
\end{tabular}
\end{table}

\subsection{Linear vs.\ Quadratic Scaling}

We compare two functional forms for the angle-adaptive target: linear scaling ($t = s \cdot \alpha \cdot r_{\mathrm{id}}$) versus quadratic scaling ($t = s^2 \cdot \alpha \cdot r_{\mathrm{id}}$) on CIFAR-10 over 3 seeds. Both variants produce nearly identical results: near-OOD KNN AUROC is $91.29 \pm 0.08$ (linear) vs.\ $91.21 \pm 0.05$ (quadratic), while MSP is $89.76$ vs.\ $89.83$; all differences fall within error bars. We default to quadratic ($s^2$) for its smoother modulation near $s = 0$.

\subsection{Contrastive Variant on CIFAR-100}
\label{app:contrastive_c100}

The contrastive margin that achieves the best CIFAR-10 results catastrophically fails on CIFAR-100: the average proxy drops from $0.780$ to $0.660$. Even a gentle $\lambda_{\mathrm{id}} = 0.1$ still hurts ($-0.025$ MSP proxy). We attribute this to the contrastive ID norm push interfering with the weaker neural collapse geometry of 100 classes. We report \method without the contrastive margin on CIFAR-100.

\subsection{The Logit-Feature Tradeoff, Quantified}

\begin{table}[h]
\centering
\caption{The logit-feature tradeoff on CIFAR-10 near-OOD (ResNet-18). $\Delta$ vs Vanilla. \method (ours, base) improves both MSP and KNN. The contrastive variant trades KNN for stronger logit scores.}
\label{tab:tradeoff}
\small
\begin{tabular}{@{}l cccc@{}}
\toprule
Method & $\Delta$ MSP & $\Delta$ Energy & $\Delta$ KNN & $\Delta$ Norm/Euc \\
\midrule
UniformNorm       & $+3.80$ & $+3.79$ & $-0.30$ & Broken \\
Phase2 (RadiusSSL) & $+3.50$ & $+3.20$ & $-0.30$ & Broken \\
\method (ours, base)     & $+1.39$ & $+2.25$ & $\mathbf{+0.47}$ & $+0.82$ \\
\method (contrastive) & $+2.27$ & $+4.72$ & $-1.60$ & $+1.55$ \\
\bottomrule
\end{tabular}
\end{table}

%% ========================================================
\section{Composability with Existing Methods}
\label{app:composability}

\method operates on feature geometry rather than logit distributions, making it orthogonal to methods that regularize outputs. Table~\ref{tab:composability} summarizes results when combining \method with OE. On CIFAR-10 near-OOD, OE alone achieves strong MSP ($94.91$) but degrades KNN from $90.64$ to $88.93$. The combination OE+\method achieves $95.02$ on MSP and $94.79$ on KNN ($+4.15$ over vanilla, $+5.86$ over OE alone). On CIFAR-100, OE+\method stabilizes KNN at $81.07$ (the highest across all configurations) with much lower variance.

\begin{table}[h]
\centering
\caption{Composability: OE, \method, and OE+\method on near-OOD detection (ResNet-18, 200 epochs, 3 seeds). AUROC ($\uparrow$) / FPR@95 ($\downarrow$).}
\label{tab:composability}
\small
\begin{tabular}{@{}l cc cc cc@{}}
\toprule
& \multicolumn{2}{c}{MSP} & \multicolumn{2}{c}{Energy} & \multicolumn{2}{c}{KNN} \\
\cmidrule(lr){2-3} \cmidrule(lr){4-5} \cmidrule(lr){6-7}
Method & AUROC & FPR & AUROC & FPR & AUROC & FPR \\
\midrule
\multicolumn{7}{l}{\textit{CIFAR-10}} \\
Vanilla       & 88.03 & 48.18 & 87.58 & 61.32 & 90.64 & 34.00 \\
\method (ours) & 89.42 & 38.78 & 89.83 & 44.19 & 91.11 & 31.02 \\
OE            & 94.91 & 19.71 & 93.44 & 24.90 & 88.93 & 42.16 \\
OE+\method (ours) & \textbf{95.02} & \textbf{18.01} & \textbf{95.28} & \textbf{18.21} & \textbf{94.79} & \textbf{19.29} \\
\midrule
\multicolumn{7}{l}{\textit{CIFAR-100}} \\
Vanilla       & 80.27 & 58.03 & 79.38 & 58.08 & 77.75 & 59.89 \\
\method (ours) & 80.43 & 54.24 & 81.19 & 54.16 & 80.71 & 59.23 \\
OE            & \textbf{87.09} & 33.51 & \textbf{86.85} & \textbf{33.70} & 79.55 & \textbf{50.73} \\
OE+\method (ours) & 86.28 & \textbf{33.46} & 85.76 & 34.06 & \textbf{81.07} & 53.45 \\
\bottomrule
\end{tabular}
\end{table}

%% ========================================================
\section{OE Fairness: Per-Split Breakdown}
\label{app:oe_fairness}

\begin{table}[h]
\centering
\caption{Per-split near-OOD AUROC on CIFAR-100 (MSP scorer, ResNet-18). OE-TIN achieves near-perfect TIN detection because the auxiliary TIN-597 set shares ImageNet's covariate distribution with the TIN-200 test split, an instance of the auxiliary-test proximity effect documented by \citet{wang2024dissecting}. Replacing TIN-597 with Places365 reduces but does not eliminate this advantage.}
\label{tab:oe_fairness}
\small
\begin{tabular}{@{}l ccc@{}}
\toprule
Method & CIFAR-10 split & TIN split & Near-OOD avg \\
\midrule
Vanilla & 77.56 & 80.86 & 79.21 \\
\method (ours) & \textbf{78.45} & 82.41 & 80.43 \\
OE (TIN-597)      & 77.38 & 99.92$^\dagger$ & 87.09 \\
OE (Places365)    & 75.05 & 95.20 & 85.13 \\
OE+\method (ours) & 77.38 & 99.92$^\dagger$ & 88.65 \\
\bottomrule
\multicolumn{4}{l}{\scriptsize $^\dagger$Near-perfect due to distributional overlap between TIN-597 auxiliary data and TIN-200 test set.}
\end{tabular}
\end{table}

%% ========================================================
\section{Far-OOD Results}
\label{app:farood}

Tables~\ref{tab:cifar10_farood} and~\ref{tab:cifar100_farood} report per-dataset far-OOD results for CIFAR-10 and CIFAR-100, respectively.

\begin{table}[h]
\centering
\caption{Far-OOD detection on CIFAR-10 (ResNet-18, 200 epochs). AUROC ($\uparrow$) / FPR@95 ($\downarrow$).}
\label{tab:cifar10_farood}
\small
\begin{tabular}{@{}ll cc cc cc@{}}
\toprule
& & \multicolumn{2}{c}{MSP} & \multicolumn{2}{c}{EBO} & \multicolumn{2}{c}{KNN} \\
\cmidrule(lr){3-4} \cmidrule(lr){5-6} \cmidrule(lr){7-8}
Dataset & Method & AUC & FPR & AUC & FPR & AUC & FPR \\
\midrule
\multirow{2}{*}{MNIST}
  & Vanilla & 93.17 & 24.25 & 93.59 & 24.23 & 92.91 & 21.54 \\
  & \method (ours) & 91.86 & 24.79 & 92.76 & 26.66 & \textbf{93.78} & \textbf{20.54} \\
\midrule
\multirow{2}{*}{SVHN}
  & Vanilla & 92.60 & 24.09 & 92.82 & 23.86 & 93.10 & 21.01 \\
  & \method (ours) & 91.36 & 25.40 & 91.87 & 29.00 & 92.71 & 21.69 \\
\midrule
\multirow{2}{*}{Textures}
  & Vanilla & 87.16 & 47.34 & 87.29 & 47.57 & 88.54 & 36.01 \\
  & \method (ours) & \textbf{90.51} & \textbf{28.88} & \textbf{90.46} & \textbf{37.77} & \textbf{93.79} & \textbf{19.91} \\
\midrule
\multirow{2}{*}{Places365}
  & Vanilla & 89.33 & 39.54 & 89.51 & 39.87 & 89.51 & 35.87 \\
  & \method (ours) & \textbf{90.09} & \textbf{34.98} & \textbf{90.77} & \textbf{39.00} & \textbf{91.91} & \textbf{28.76} \\
\midrule
\multirow{2}{*}{\textit{Average}}
  & Vanilla & 90.57 & 33.81 & 90.80 & 33.88 & 91.02 & 28.61 \\
  & \method (ours) & 90.95 & \textbf{28.51} & \textbf{91.47} & 33.11 & \textbf{93.05} & \textbf{22.73} \\
\bottomrule
\end{tabular}
\end{table}

\begin{table}[h]
\centering
\caption{Far-OOD detection on CIFAR-100 (ResNet-18, 200 epochs). AUROC ($\uparrow$) / FPR@95 ($\downarrow$).}
\label{tab:cifar100_farood}
\small
\begin{tabular}{@{}ll cc cc@{}}
\toprule
& & \multicolumn{2}{c}{MSP} & \multicolumn{2}{c}{KNN} \\
\cmidrule(lr){3-4} \cmidrule(lr){5-6}
Dataset & Method & AUC & FPR & AUC & FPR \\
\midrule
\multirow{2}{*}{MNIST}
  & Vanilla & 76.17 & 51.39 & 78.66 & 51.11 \\
  & \method (ours) & \textbf{78.75} & \textbf{50.83} & \textbf{84.15} & \textbf{43.46} \\
\midrule
\multirow{2}{*}{SVHN}
  & Vanilla & 77.50 & 61.15 & 78.31 & 57.84 \\
  & \method (ours) & \textbf{82.54} & \textbf{49.36} & \textbf{88.78} & \textbf{41.39} \\
\midrule
\multirow{2}{*}{Textures}
  & Vanilla & 74.47 & 67.49 & 75.83 & 61.03 \\
  & \method (ours) & \textbf{77.70} & \textbf{60.22} & \textbf{82.06} & \textbf{58.70} \\
\midrule
\multirow{2}{*}{Places365}
  & Vanilla & 77.82 & 60.21 & 77.01 & 61.89 \\
  & \method (ours) & \textbf{79.82} & \textbf{55.17} & \textbf{80.77} & \textbf{57.27} \\
\midrule
\multirow{2}{*}{\textit{Average}}
  & Vanilla & 76.49 & 60.06 & 77.45 & 57.97 \\
  & \method (ours) & \textbf{79.70} & \textbf{53.89} & \textbf{83.94} & \textbf{50.20} \\
\bottomrule
\end{tabular}
\end{table}

%% ========================================================
\section{Baseline Selection and Coverage Map}
\label{app:coverage}

The OpenOOD v1.5 benchmark contains a large number of training-based OOD detection methods. The main tables (Tables~\ref{tab:cifar10_nearood}--\ref{tab:cifar100_nearood}) include one representative per design paradigm. Table~\ref{tab:coverage_map} lists every training-based method in the benchmark with its paradigm and inclusion status. Methods are excluded only when (i) the paradigm is already covered by an included method, or (ii) the architecture is incompatible with the cross-scorer evaluation that supports \method's universal-compatibility claim.

\begin{table}[h]
\centering
\caption{Baseline coverage map for OpenOOD v1.5 training-based methods. Cross-scorer evaluation (MSP, Energy, ODIN, KNN, ReAct, ASH, Scale) requires a standard linear classifier head. Methods with non-standard heads cannot produce all seven scorers without custom inference.}
\label{tab:coverage_map}
\small
\begin{tabular}{@{}l l l@{}}
\toprule
Paradigm & Included & Excluded (rationale) \\
\midrule
Cross-entropy baseline      & Vanilla (100/200/300/400 ep) & --- \\
Logit normalization         & LogitNorm                    & T2FNorm (same paradigm) \\
Synthetic outlier (virtual) & VOS, NPOS                    & ARPL (reciprocal-points head, no standard logits) \\
                            &                              & OpenGAN (training does not converge on CIFAR setup) \\
Contrastive feature learning & CIDER                       & --- \\
Auxiliary self-supervision  & ConfBranch, RotPred          & --- \\
Real auxiliary OOD data     & OE, MixOE                    & --- \\
Null-space training         & PFS                          & --- \\
Hierarchical classification & ---                          & MOS (designed for ImageNet super-class hierarchy; \\
                            &                              & 1.71\% ID accuracy on flat-class CIFAR-100 at default hyperparameters) \\
Cosine head + temperature   & ---                          & G-ODIN (non-standard cosine head; standard scorers do not apply) \\
Adversarial dual-classifier & ---                          & MCD (two-classifier architecture; single-fc cross-scorer pipeline does not apply) \\
Pseudo-OOD via clustering   & ---                          & UDG ($N{=}1000$-cluster head; cross-scorer pipeline does not apply) \\
\bottomrule
\end{tabular}
\end{table}

\paragraph{Summary.} Of the eleven training-based paradigms in OpenOOD v1.5, \method is compared against representatives from seven. The four paradigms not represented (hierarchical classification, cosine-head, adversarial dual-classifier, clustering-head) are excluded because their architectures do not produce a standard linear-classifier output, which is required for the cross-scorer comparison that supports the universal-compatibility claim. Within paradigms that have multiple methods, we include those with standard scoring heads (VOS, NPOS, LogitNorm) and exclude variants with non-standard heads (ARPL, OpenGAN, T2FNorm). MOS is excluded on convergence grounds: it is designed for hierarchical class structures and the OpenOOD-default training does not converge on flat-class CIFAR.

%% ========================================================
\section{Additional Experimental Results}
\label{app:additional}

\subsection{Additional Scorers}
\label{app:scorers}

\begin{table}[h]
\centering
\caption{ReAct and ASH scorers on CIFAR-10 (ResNet-18, 3 seeds). \method (ours, base) is the non-contrastive variant ($\lambda_{\mathrm{id}}=0$); the contrastive variant used in the main tables (Table~\ref{tab:scorer_robustness}) achieves $92.06$ ReAct / $89.87$ ASH near-OOD AUROC. ASH's percentile-based activation truncation interacts poorly with the base variant's norm distribution, but the contrastive margin resolves this.}
\label{tab:additional_scorers}
\small
\begin{tabular}{@{}l cc cc@{}}
\toprule
& \multicolumn{2}{c}{Near-OOD} & \multicolumn{2}{c}{Far-OOD} \\
\cmidrule(lr){2-3} \cmidrule(lr){4-5}
Scorer / Method & AUROC & FPR & AUROC & FPR \\
\midrule
\multicolumn{5}{l}{\textit{ReAct}} \\
Vanilla & $87.11 \pm 0.61$ & $63.56 \pm 7.33$ & $90.42 \pm 1.41$ & $44.90 \pm 8.37$ \\
\method (ours, base) & $\mathbf{89.57 \pm 0.17}$ & $\mathbf{52.23 \pm 2.18}$ & $\mathbf{92.34 \pm 0.73}$ & $\mathbf{31.77 \pm 3.62}$ \\
\midrule
\multicolumn{5}{l}{\textit{ASH}} \\
Vanilla & $\mathbf{75.27 \pm 1.04}$ & $\mathbf{86.78 \pm 1.82}$ & $\mathbf{78.49 \pm 2.58}$ & $\mathbf{79.03 \pm 4.22}$ \\
\method (ours, base) & $72.06 \pm 0.52$ & $89.73 \pm 0.16$ & $73.99 \pm 1.56$ & $86.34 \pm 1.14$ \\
\bottomrule
\end{tabular}
\end{table}

\subsection{Additional Backbones}
\label{app:backbones}

\begin{table}[h]
\centering
\caption{Near-OOD detection with WRN-28-10 backbone (200 epochs). AUROC ($\uparrow$) / FPR@95 ($\downarrow$). On CIFAR-10, \method base achieves the strongest KNN ($92.04$, edging Vanilla WRN's $92.00$); the contrastive variant ($\lambda_{\mathrm{id}}=0.5$) gains $+2.18$ MSP and $+4.51$ Energy over Vanilla WRN at a small $-0.18$ KNN cost. On CIFAR-100, \method beats Vanilla WRN on all three scorers (no contrastive variant: the 100-class ETF is too fragile to support the contrastive margin).}
\label{tab:wrn_results}
\small
\begin{tabular}{@{}l l cc cc cc@{}}
\toprule
& & \multicolumn{2}{c}{MSP} & \multicolumn{2}{c}{Energy} & \multicolumn{2}{c}{KNN} \\
\cmidrule(lr){3-4} \cmidrule(lr){5-6} \cmidrule(lr){7-8}
Dataset & Method & AUC & FPR & AUC & FPR & AUC & FPR \\
\midrule
\multirow{3}{*}{\shortstack[l]{CIFAR-10}} & Vanilla (R18) & 88.03 & 48.18 & 87.58 & 61.32 & 90.64 & 34.00 \\
& Vanilla (WRN) & $89.13_{\pm.23}$ & $46.74_{\pm2.9}$ & $88.85_{\pm.44}$ & $59.55_{\pm3.2}$ & $92.00_{\pm.01}$ & $30.27_{\pm.02}$ \\
& \method (ours, WRN, base) & $89.36_{\pm.11}$ & $48.94_{\pm.94}$ & $89.52_{\pm.15}$ & $57.19_{\pm.70}$ & $\mathbf{92.04}_{\pm.09}$ & $\mathbf{29.38}_{\pm.18}$ \\
& \method (ours, WRN, contr.) & $\mathbf{91.31}_{\pm.16}$ & $\mathbf{28.22}_{\pm.21}$ & $\mathbf{93.36}_{\pm.18}$ & $\mathbf{26.33}_{\pm.76}$ & $91.82_{\pm.06}$ & $31.58_{\pm.56}$ \\
\midrule
\multirow{3}{*}{\shortstack[l]{CIFAR-100}} & Vanilla (R18) & 80.27 & 58.03 & 79.38 & 58.08 & 77.75 & 63.97 \\
& Vanilla (WRN) & $81.65_{\pm.32}$ & $53.11_{\pm.17}$ & $82.19_{\pm.29}$ & $54.11_{\pm.16}$ & $82.56_{\pm.11}$ & $54.17_{\pm.22}$ \\
& \method (ours, WRN)  & $82.53_{\pm.18}$ & $51.98_{\pm1.0}$ & $82.85_{\pm.23}$ & $53.07_{\pm1.0}$ & $\mathbf{82.86}_{\pm.20}$ & $\mathbf{52.74}_{\pm.66}$ \\
\bottomrule
\end{tabular}
\end{table}

\paragraph{ViT-B/16.}
Table~\ref{tab:vit_results} reports ViT-B/16 results on CIFAR-10 ($32{\times}32$ upscaled to $224{\times}224$). \method does not improve over the vanilla ViT, which itself achieves low ID accuracy ($74.43\%$) due to ViT-B/16 being over-parameterized for 50K $32{\times}32$ images. Two factors explain this: (i) LayerNorm before the classifier head erases the norm variation that \method's geometry-preserving loss relies on, and (ii) the low ID accuracy indicates that the model has not reached the neural collapse regime, violating the ETF assumptions underlying our adaptive targeting.

\begin{table}[h]
\centering
\caption{OOD detection with ViT-B/16 on CIFAR-10 ($32{\times}32 \to 224{\times}224$, single seed). The ID accuracy of $74.43\%$ shows ViT-B/16 has not reached the neural collapse regime on $50K$ low-resolution training images, which (together with pre-classifier LayerNorm) is the mechanism we attribute the lack of \method improvement to. We report Vanilla ResNet-18 as a same-dataset reference baseline; we did not train a matched Vanilla ViT-B/16 baseline at this setting due to compute constraints, so the claim that \method does not improve over Vanilla on this regime should be read as a qualitative diagnostic rather than a head-to-head comparison. Far-OOD FPR omitted (different evaluation protocol).}
\label{tab:vit_results}
\small
\begin{tabular}{@{}l cc cc c@{}}
\toprule
& \multicolumn{2}{c}{Near-OOD} & \multicolumn{2}{c}{Far-OOD} & \\
\cmidrule(lr){2-3} \cmidrule(lr){4-5}
Method & AUROC & FPR & AUROC & FPR & ID Acc \\
\midrule
Vanilla (ResNet-18) & 88.03 & 48.18 & 90.73 & --- & 95.06 \\
\method (ours, ViT-B/16)  & 67.85 & 77.92 & 65.62 & 75.91 & 74.43 \\
\bottomrule
\end{tabular}
\end{table}

\subsection{Far-OOD Detection}
\label{app:farood_detection}

Table~\ref{tab:farood} reports far-OOD detection (MNIST, SVHN, Textures, Places365 for CIFAR; iNaturalist, Textures, OpenImage-O for ImageNet-200). \method improves over vanilla on CIFAR-10 far-OOD across all scorers ($+0.2$ MSP, $+0.15$ Energy). On CIFAR-100, \method improves far-OOD MSP ($80.03$ vs $78.55$) and Energy ($81.14$ vs $79.98$) over vanilla at 100 epochs. OE+\method achieves the strongest far-OOD results on both datasets.

\begin{table}[h]
\centering
\caption{Far-OOD detection AUROC ($\uparrow$). \method improves or matches Vanilla on CIFAR-10 and CIFAR-100; on ImageNet-200, \method underperforms Vanilla on logit scorers (consistent with the near-OOD picture in Table~\ref{tab:imagenet200}; NC weakens at $C=200$, $p=512$). OE+\method achieves the strongest far-OOD detection on the CIFAR datasets while maintaining near-OOD KNN (Table~\ref{tab:cifar10_nearood}).}
\label{tab:farood}
\small
\begin{tabular}{@{}l l ccc@{}}
\toprule
Dataset & Method & MSP & Energy & KNN \\
\midrule
\multirow{3}{*}{CIFAR-10} & Vanilla & 90.57 & 90.80 & 91.02 \\
& \method (ours) & 90.95 & 91.47 & \textbf{93.05} \\
& OE+\method (ours) & \textbf{96.81} & \textbf{97.02} & \textbf{96.86} \\
\midrule
\multirow{3}{*}{CIFAR-100} & Vanilla (200 ep) & 76.49 & --- & 77.45 \\
& \method (ours, 200 ep) & 79.70 & --- & \textbf{83.94} \\
& OE+\method (ours) & \textbf{91.34} & \textbf{92.64} & \textbf{83.31} \\
\midrule
\multirow{2}{*}{ImageNet-200} & Vanilla & \textbf{90.13} & \textbf{90.86} & 93.05 \\
& \method (ours) & 89.15 & 89.79 & \textbf{92.07} \\
\bottomrule
\multicolumn{5}{l}{\scriptsize CIFAR-10/100 Vanilla and \method rows are averages across far-OOD datasets (Tables~\ref{tab:cifar10_farood}--\ref{tab:cifar100_farood}); CIFAR-100 Energy not computed at 200 ep.}
\end{tabular}
\end{table}

\subsection{ImageNet-200 Cross-Scorer Detail}
\label{app:imagenet200_xscorer}

Table~\ref{tab:imagenet200_xscorer} extends the ImageNet-200 evaluation (Table~\ref{tab:imagenet200}) to all seven scorers, and adds two no-auxiliary-data baselines (LogitNorm, VOS). The picture is mixed: \method underperforms Vanilla by $1$--$3$ AUROC points on every logit scorer, consistent with the weakened neural collapse at $200$ classes with $p=512$ noted in Section~\ref{sec:main_results}. \method does match or improve KNN over Vanilla. The headline contrast is with LogitNorm, which catastrophically fails on Energy ($29.67$), ReAct ($26.53$), and ASH ($57.88$), while \method's worst scorer on ImageNet-200 is ODIN at $77.93$. \method's range across the seven scorers ($77.93$--$83.48$) is comparable to VOS ($78.25$--$83.35$) and Vanilla ($80.27$--$84.84$), and far tighter than LogitNorm ($26.53$--$81.02$, range $54.5$).

\begin{table}[h]
\centering
\caption{ImageNet-200 near-OOD detection across seven scorers (ResNet-18, 3 seeds). AUROC ($\uparrow$). LogitNorm catastrophically fails on Energy, ReAct, and ASH, while \method underperforms Vanilla on logit scorers without catastrophic failure on any.}
\label{tab:imagenet200_xscorer}
\small
\begin{tabular}{@{}l ccccccc@{}}
\toprule
Method & MSP & Energy & ODIN & KNN & ReAct & ASH & Scale \\
\midrule
Vanilla    & $\mathbf{83.34}$ & $\mathbf{82.50}$ & $\mathbf{80.27}$ & $80.40$ & $\mathbf{81.87}$ & $\mathbf{82.38}$ & $\mathbf{84.84}$ \\
LogitNorm  & $81.02$ & $29.67$$^\ddagger$ & $78.80$ & $78.15$ & $26.53$$^\ddagger$ & $57.88$ & $74.03$ \\
VOS        & $81.55$ & $80.75$ & $78.25$ & $79.33$ & $79.45$ & $79.49$ & $83.35$ \\
\method (ours) & $82.28$ & $81.56$ & $77.93$ & $\mathbf{80.59}$ & $80.62$ & $81.13$ & $83.48$ \\
\bottomrule
\multicolumn{8}{l}{\scriptsize $^\ddagger$LogitNorm catastrophic failures: Energy $29.67$, ReAct $26.53$, both worse than random.}
\end{tabular}
\end{table}

\subsection{Null-Space Decomposition}
\label{app:nullspace}

Proposition~\ref{prop:tradeoff} predicts that the null space of the classifier $W$ carries norm signal that is invisible to logit-based scorers. Table~\ref{tab:geometry_decomp} reports the empirical decomposition of ID and OOD feature energy between $\col(W^\top)$ and $\nul(W)$ on a vanilla ResNet-18 CIFAR-10 backbone ($C=10$, $p=512$): $72.2\%$ of ID feature energy lies in $\col(W^\top)$, with the remaining $27.8\%$ in the null space available to feature-space scorers. PFS shifts this distribution dramatically, ejecting $92.7\%$ of OOD energy into $\nul(W)$.

\subsection{Feature Norm Distributions}
\label{app:norms}

\begin{figure}[h]
\centering
\includegraphics[width=\textwidth]{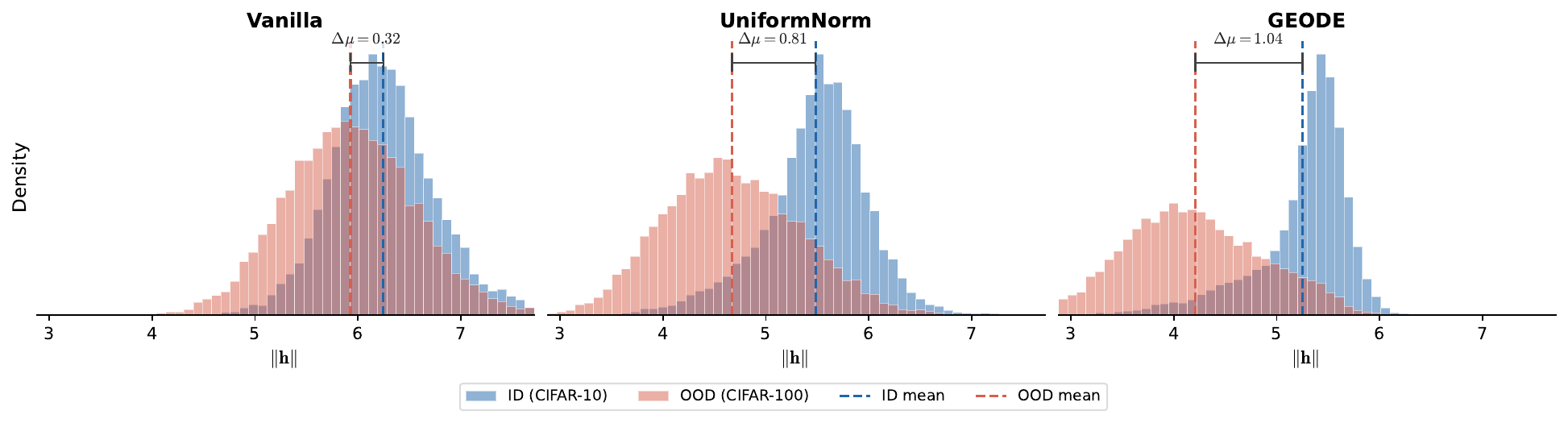}
\caption{Feature norm distributions for ID (CIFAR-10, blue) and near-OOD (CIFAR-100, red) under three training methods (ResNet-18). Vanilla: overlapping (gap $0.32$). UniformNorm: moderate separation (gap $0.81$). \method: best separation (gap $1.04$) with preserved OOD distribution shape.}
\label{fig:norm_dist}
\end{figure}

\subsection{Training Curves}
\label{app:curves}

\begin{figure}[h]
\centering
\includegraphics[width=\textwidth]{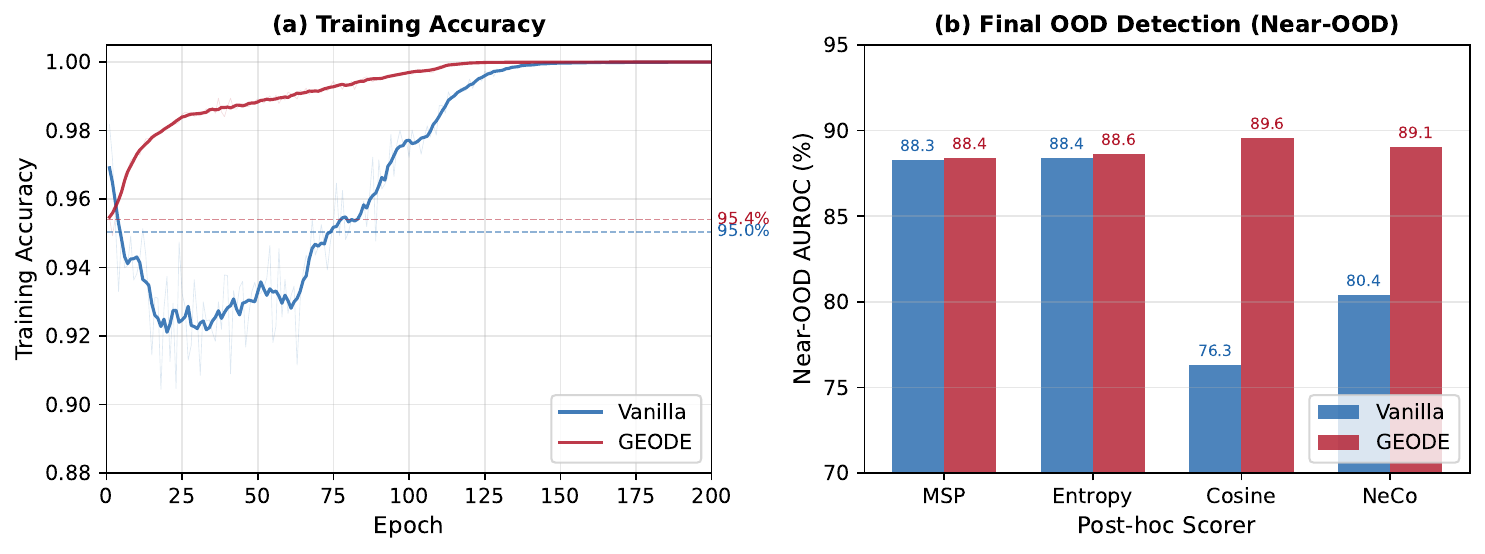}
\caption{Training dynamics on CIFAR-10 (ResNet-18, 200 epochs) for the \method \emph{base} variant ($\lambda_{\mathrm{id}}=0$). (a)~Per-epoch training accuracy (smoothed; raw values shown faintly). (b)~Near-OOD AUROC for four post-hoc scorers at convergence. The base variant tracks Vanilla on logit-based scorers and improves feature-based scorers; the contrastive variant ($\lambda_{\mathrm{id}}=0.5$) used in the main results (Table~\ref{tab:cifar10_nearood}) shifts further toward logit-side gains at a small KNN cost.}
\label{fig:training_curves}
\end{figure}

\subsection{Full Hyperparameter Sensitivity}
\label{app:sensitivity}

\begin{table}[h]
\centering
\caption{Hyperparameter sensitivity: near-OOD AUROC on CIFAR-10 as a function of $\alpha$ (ResNet-18, 200 epochs, $\rho=0.2$).}
\label{tab:sensitivity_c10}
\small
\begin{tabular}{@{}l cccc@{}}
\toprule
$\alpha$ & 0.2 & 0.3 & 0.5 & 0.7 \\
\midrule
MSP   & \textbf{88.95} & 88.41 & 88.16 & 88.45 \\
Entropy & \textbf{89.15} & 88.62 & 88.36 & 88.65 \\
Cos   & \textbf{89.87} & 89.56 & 89.68 & 89.66 \\
Euc   & \textbf{89.74} & 89.43 & 89.63 & 89.62 \\
NeCo  & \textbf{89.94} & 89.40 & 89.27 & 88.97 \\
\midrule
ID Acc & 95.30 & \textbf{95.46} & 95.33 & 95.43 \\
\bottomrule
\end{tabular}
\end{table}

\begin{table}[h]
\centering
\caption{Hyperparameter sensitivity: near-OOD AUROC on CIFAR-100 as a function of $\alpha$ (ResNet-18, 200 epochs, $\rho=0.2$).}
\label{tab:sensitivity_c100}
\small
\begin{tabular}{@{}l cccc@{}}
\toprule
$\alpha$ & 0.2 & 0.3 & 0.5 & 0.7 \\
\midrule
MSP   & 78.74 & 78.59 & \textbf{79.11} & 78.51 \\
Entropy & 79.45 & 79.26 & \textbf{79.78} & 79.25 \\
Cos   & 78.07 & 77.90 & \textbf{78.14} & 78.04 \\
Euc   & 75.27 & 75.04 & \textbf{75.19} & 75.23 \\
NeCo  & 75.11 & 75.06 & \textbf{75.73} & 75.20 \\
\midrule
ID Acc & \textbf{77.51} & \textbf{77.51} & 77.32 & 77.54 \\
\bottomrule
\end{tabular}
\end{table}

\subsection{CIFAR-100 Detailed Results}
\label{app:cifar100}

\begin{table}[h]
\centering
\caption{Per-dataset CIFAR-100 results for \method (ResNet-18, 200 epochs). AUROC ($\uparrow$) / FPR@95 ($\downarrow$).}
\label{tab:cifar100_detailed}
\small
\begin{tabular}{@{}l cc cc cc cc@{}}
\toprule
& \multicolumn{2}{c}{MSP} & \multicolumn{2}{c}{Energy} & \multicolumn{2}{c}{ODIN} & \multicolumn{2}{c}{KNN} \\
\cmidrule(lr){2-3} \cmidrule(lr){4-5} \cmidrule(lr){6-7} \cmidrule(lr){8-9}
Dataset & AUC & FPR & AUC & FPR & AUC & FPR & AUC & FPR \\
\midrule
\multicolumn{9}{l}{\textit{Near-OOD}} \\
CIFAR-10   & 78.45 & 58.67 & 79.17 & 59.21 & 77.82 & 59.89 & 77.63 & 69.69 \\
TIN        & 82.41 & 49.82 & 83.22 & 49.11 & 81.66 & 52.39 & 83.80 & 48.77 \\
\textit{Average} & 80.43 & 54.24 & 81.19 & 54.16 & 79.74 & 56.14 & 80.71 & 59.23 \\
\midrule
\multicolumn{9}{l}{\textit{Far-OOD}} \\
MNIST      & 78.75 & 50.83 & 80.77 & 48.00 & 86.59 & 40.14 & 84.15 & 43.46 \\
SVHN       & 82.54 & 49.36 & 84.49 & 44.39 & 78.86 & 55.56 & 88.78 & 41.39 \\
Textures   & 77.70 & 60.22 & 79.58 & 59.89 & 79.71 & 61.07 & 82.06 & 58.70 \\
Places365  & 79.82 & 55.17 & 80.48 & 54.58 & 79.99 & 56.51 & 80.77 & 57.27 \\
\textit{Average} & 79.70 & 53.89 & 81.33 & 51.71 & 81.29 & 53.32 & 83.94 & 50.20 \\
\bottomrule
\end{tabular}
\end{table}

\subsection{Computational Cost}
\label{app:cost}

\begin{table}[h]
\centering
\caption{Training wall-clock time (seconds per epoch / total hours) on a single RTX 2080 Ti GPU with batch size 256. \method adds negligible overhead vs.\ Vanilla in our measurements; the per-epoch numbers are noisy because of shared-cluster I/O variance, which on CIFAR-10 makes \method appear marginally faster than Vanilla. The CIFAR-100 row, where the two methods were measured back-to-back, shows the true overhead is within $\pm 2\%$. All methods train for 200 epochs.}
\label{tab:cost}
\begin{tabular}{@{}lcc@{}}
\toprule
Method & CIFAR-10 & CIFAR-100 \\
\midrule
Vanilla (CE only)  & 56 s/ep (3.1 h) & 69 s/ep (3.8 h) \\
\method (ours)     & 42 s/ep (2.4 h) & 68 s/ep (3.8 h) \\
WRN + \method (ours) & 91 s/ep (5.1 h) & 95 s/ep (5.3 h) \\
\bottomrule
\end{tabular}
\end{table}

\subsection{Gradient Analysis}
\label{app:gradients}

We analyze how the \method OOD norm loss $\mathcal{L} = (\norm{\mathbf{h}} - t)^2$ with adaptive target $t = s^2 \cdot \alpha \cdot r_{\mathrm{id}}$ shapes the feature space, where $s = \max_c \cosim(\mathbf{h}, \boldsymbol{\mu}_c)$ and we write $\mathbf{h} = \mathbf{h}^{\mathrm{ood}}$ for brevity. Let $r = \norm{\mathbf{h}}$, $\hat{\mathbf{h}} = \mathbf{h}/r$, and let $c^* = \arg\max_c \cosim(\mathbf{h}, \boldsymbol{\mu}_c)$ denote the nearest class mean. Define the unit class mean direction $\hat{\boldsymbol{\mu}} = \boldsymbol{\mu}_{c^*}/\norm{\boldsymbol{\mu}_{c^*}}$.

The gradient decomposes as $\partial \mathcal{L}/\partial \mathbf{h} = \mathbf{g}_{\mathrm{rad}} + \mathbf{g}_{\mathrm{tan}}$, where:
\begin{align}
    \mathbf{g}_{\mathrm{rad}} &= 2(r - t)\,\hat{\mathbf{h}}, \label{eq:grad_radial} \\
    \mathbf{g}_{\mathrm{tan}} &= -\frac{4s(r - t) \cdot \alpha \cdot r_{\mathrm{id}}}{r\norm{\boldsymbol{\mu}_{c^*}}} \left(\boldsymbol{\mu}_{c^*} - s\norm{\boldsymbol{\mu}_{c^*}}\,\hat{\mathbf{h}}\right). \label{eq:grad_tangential}
\end{align}
The radial component $\mathbf{g}_{\mathrm{rad}}$ adjusts the feature norm toward the target $t$. The tangential component $\mathbf{g}_{\mathrm{tan}}$ rotates the feature direction; when $\mathbf{h}$ is aligned with a class mean ($s \to 1$), $\mathbf{g}_{\mathrm{tan}} \to \mathbf{0}$, so the loss gradient is purely radial near class means, preserving angular arrangement while regularizing norms.

%% ========================================================
\section{Design Intuition}
\label{app:analysis_details}

The adaptive target connects the OE geometric analysis (Section~\ref{sec:oe_analysis}) to the neural collapse framework: under NC, ID features converge to class means on a simplex ETF, so conditioning the OOD norm target on cosine similarity to these means respects the angular structure that defines class boundaries. Near-boundary OOD features ($s \to 1$) retain meaningful norm, preserving local distance relationships that KNN relies on, while far-from-class features ($s \to 0$) are pushed to the origin for logit-based discrimination. The contrastive margin amplifies this by pushing ID norms up (Eq.~\ref{eq:id_norm}), placing the ID--OOD boundary at the Q4 region identified in the OE quartile analysis (Table~\ref{tab:oe_quartile}).

%% ========================================================
\section{Future Work}
\label{app:future}

Three directions are promising: (i) synthetic generation strategies that directly target the boundary-calibration locus identified by the OE quartile analysis; (ii) extending the composability framework to other training methods beyond OE; and (iii) investigating whether the formal score-family tradeoff (Proposition~\ref{prop:tradeoff}; see also \citet{wang2025impossibility} for an independent analysis) can guide the design of training losses that are provably optimal for simultaneous near- and far-OOD detection.

\end{document}